\documentclass{article}

\usepackage{arxiv}

\usepackage[utf8]{inputenc} 
\usepackage[T1]{fontenc}    
\usepackage{hyperref}       
\usepackage{url}            
\usepackage{booktabs}       
\usepackage{amsfonts}       
\usepackage{amsmath}        
\usepackage{nicefrac}       
\usepackage{microtype}      
\usepackage{subfig}         
\usepackage{fancyhdr}      
\usepackage{lipsum}		
\usepackage{graphicx}
\usepackage{natbib}
\usepackage{doi}

\usepackage{graphicx}%
\usepackage{multirow}%
\usepackage{amsmath,amssymb,amsfonts}%
\usepackage{amsthm}%
\usepackage{mathrsfs}%
\usepackage[title]{appendix}%
\usepackage{xcolor}%
\usepackage{textcomp}%
\usepackage{manyfoot}%
\usepackage{booktabs}%
\usepackage{algorithm}%
\usepackage{algorithmicx}%
\usepackage{algpseudocode}%
\usepackage{listings}%

\usepackage{setspace}

\usepackage{natbib}
\newtheorem{theorem}{Theorem}
\newtheorem{proposition}[theorem]{Proposition}%

\title{Neuronal Group Communication for 
Efficient Neural representation}

\author{ 
{\includegraphics[scale=0.06]{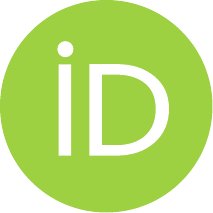}\hspace{1mm}Zhengqi Pei}\\
	Institute of Computing Technology, 
    Chinese Academy of Sciences.\\
    School of Computer Science and Technology,
    University of Chinese Academy of Sciences.\\
	Beijing, China \\
	\texttt{peizhengqi25b@ict.ac.cn} \\
    \And
    {\includegraphics[scale=0.06]{orcid.pdf}\hspace{1mm}Qingming Huang} \\
	School of Computer Science and Technology\\
	University of Chinese Academy of Sciences\\
	Beijing, China\\
	\texttt{qmhuang@ucas.ac.cn} \\
	\And
    {\includegraphics[scale=0.06]{orcid.pdf}\hspace{1mm}Shuhui Wang}\thanks{Correspondence author} \\
	Institute of Computing Technology\\
	Chinese Academy of Sciences\\
	Beijing, China \\
	\texttt{wangshuhui@ict.ac.cn} \\
}

\begin{document}

\maketitle

\begin{abstract}

The ever-increasing scale of modern neural networks has brought unprecedented performance alongside daunting challenges in efficiency and interpretability.
This paper addresses the core question of how to build large neural systems that learn efficient, modular, and interpretable representations.
We propose \textbf{Neuronal Group Communication} (NGC), a theory-driven framework that reimagines a neural network as a dynamical system of interacting neuronal groups rather than a monolithic collection of neural weights.
Instead of treating each weight as an independent trainable parameter, NGC treats weights as transient interactions between embedding-like neuronal states, with neural computation unfolding through iterative communication among groups of neurons.
This low-rank, modular representation yields compact models: 
groups of neurons exchange low-dimensional signals, enabling intra-group specialization and inter-group information sharing while dramatically reducing redundant parameters.
By drawing on dynamical systems theory, we introduce a neuronal stability metric (analogous to Lyapunov stability) that quantifies the contraction of neuron activations toward stable patterns during sequence processing.
Using this metric, we reveal that emergent reasoning capabilities correspond to an external driving force or ``potential'', which nudges the neural dynamics away from trivial trajectories while preserving stability.
Empirically, we instantiate NGC in large language models (LLMs) and demonstrate improved performance on complex reasoning benchmarks under moderate compression.
NGC consistently outperforms standard low-rank approximations and cross-layer basis-sharing methods at comparable compression rates. 
We conclude by discussing the broader implications of NGC, including how structured neuronal group dynamics might relate to generalization in high-dimensional learning systems.

\end{abstract}

\section{Introduction}

Deep neural networks have achieved remarkable success across domains, but this progress has come at the cost of massive model scales.
State-of-the-art models often contain tens or hundreds of billions of parameters~\citep{zhao2023survey,naveed2025comprehensive},
making them resource-intensive to train and deploy and challenging to interpret.
This reality has spurred extensive research on model compression and efficient representations~\citep{wang2024svd,wangbasis}.
A range of techniques have been explored to reduce the size of the model while preserving performance, {\it e.g.}, weight pruning~\citep{liu2018rethinking}, quantization~\citep{xu2018deep, nagel2021white}, knowledge distillation~\citep{hinton2015distilling}, and low-rank approximation~\citep{yuan2023asvd,wang2024svd}.
In particular, factoring weight matrices via singular value decomposition (SVD) or related methods can compress models with minimal accuracy loss. 
For example, recent SVD-based compression methods such as SVD-LLM~\citep{wang2024svd} and Dobi-SVD~\citep{qinsidobi} achieve state-of-the-art results by truncating small singular values and fine-tuning factorized weights.
Similarly, cross-layer low-rank basis sharing has been explored to further reduce redundancy by using shared basis vectors across multiple layers~\citep{wangbasis}. 
These advances underscore a broader insight: 
much of an LLM’s parameter space is redundant, and a large portion of model behavior can be captured in a far lower-dimensional subspace.

Parallel to these engineering efforts, researchers are also investigating the internal organization and dynamics of large neural models.
There is growing evidence that such models are not homogeneous ``black boxes'', but exhibit an internal modular structure with specialized groups of neurons~\citep{arbib1998neural,meunier2010modular,perin2011synaptic,xiao2024configurable}.
For example, studies of Transformer networks~\citep{vaswani2017attention} have found that only a small fraction of neurons in feed-forward layers are significantly activated for any given input, indicating sparse and specialized functionality~\citep{qiu2024unlocking}. 
Moreover, neurons with similar functions tend to co-activate, effectively forming functional groups or clusters~\citep{yuste1995neuronal,eickhoff2011co,zhang2023emergent}. 
This phenomenon, often termed ``emergent modularity'', reflects how modular organization can spontaneously arise even in standard (non-explicitly modular) networks~\citep{bullmore2009complex,meunier2010modular}.
In other words, large-scale models appear internally organized into neuronal communities that specialize in particular aspects of the input or subtasks of the problem. 
Such observations resonate with long-standing concepts in neuroscience and cognitive science that complex systems achieve efficiency through modularity and sparse communication between subcomponents~\citep{goyal2019recurrent,dobs2022brain}. 
They also align with the success of explicit modular architectures like Mixture-of-Experts, which allocate distinct ``expert'' sub-networks to different inputs~\citep{gururangan2022demix,fedus2022switch}. 
Such modular or group-wise structures confer benefits in adaptability, interpretability, and robustness~\citep{qiu2024unlocking}.

Our work is motivated by uniting these two perspectives: 
the need for efficient neural representations in large models and the reality of emergent neuronal modularity. 
This leads to a fundamental question: 
\textbf{Can we design large neural systems that are both more efficient in parameter usage and more interpretable in their internal computation by leveraging a modular, group-based organization?}
In particular, rather than structuring models around traditional layer-wise weight matrices, 
can we reorganize them around their \emph{emergent neuronal groups}? 
Furthermore, can principles from dynamical systems theory help us analyze and guide such a reorganization? 
Addressing these questions requires rethinking the basic building blocks of neural computation in high-dimensional models.
To this end, we propose \textbf{Neuronal Group Communication} (NGC) as a principled framework for efficient and dynamic neural representation.
NGC reframes a neural network as a collection of interacting neuronal groups, rather than a monolithic weight matrix.
A \emph{neuronal group} denotes a set of neurons that are strongly interconnected (or functionally related), forming a small community within the network.
Communication between neurons, which is ordinarily encoded implicitly in large weight matrices, is made explicit in NGC as dynamic information interactions between these neuronal groups.
In other words, rather than every connection being a fixed scalar weight, we view these connections as interactions mediated by shared neuronal states.
Each neuronal group can have rich internal interactions called \emph{intra-group} communication and also exchange information with other neuronal groups via \emph{inter-group} communication.
Notably, inter-group communication in NGC can represent a \emph{virtual information flow} that does not require a direct pathway in the original static neural structure;
this intuition is supported by system neuroscience, where functional interactions often transcend anatomical connectivity~\citep{kriegeskorte2008representational,friston2011functional,semedo2019cortical}.

Consequently, NGC shifts the learning focus from neural weights to neurons themselves by treating each neuron (or each small neuron community) as a trainable unit with an embedding-like intrinsic state.
In conventional deep learning, training adjusts the many weight parameters that connect neurons. 
In NGC, what we aim to learn and preserve are the intrinsic latent states of individual neurons, while the weight values that couple them are considered transient and are derived from those states.
This perspective is analogous to principles in communication networks: 
the Internet favors simple, flexible routing (``best-effort'' delivery) and pushes complexity to the end hosts, rather than hard-coding every connection~\citep{ishida1993internetworking,clark1998explicit,hunt2002review}. 
Therefore, a neural model might operate more efficiently if it does not overly commit to a fixed pattern of inter-neuron communication (specific weight values), but instead maintains adaptable neuron representations that can communicate as needed.

In practice, NGC implements this idea through low-rank factorization of weight interactions.
Any weight matrix ${\bf W}\in\mathbb{R}^{m\times n}$ can be approximately decomposed as ${\bf W} \approx {\bf A}{\bf B}^{\top}$, with ${\bf A}\in\mathbb{R}^{m\times r}$ and ${\bf B}\in\mathbb{R}^{n\times r}$ for some small $r$. 
More concretely, any entry ${\bf W}[i,j]$ can be approximated using an interactive metric $\mu$ between ${\bf A}[i]$ and ${\bf B}[j]$ without loss, {\it i.e.}, ${\bf W}[i,j]=\mu({\bf A}[i],{\bf B}[j])$ for any $i,j$.
We interpret the rows of ${\bf A}$ and ${\bf B}$ as neuronal states for the output and input neurons of ${\bf W}$, respectively. 
These $r$-dimensional state embeddings characterize neurons in a learned latent space (analogous to word embeddings in NLP). 
The original weight ${\bf W}_{ij}$ then corresponds to a communication interaction, which can be accessed through the dot product or MLP, between the state of the neuron $i$ and the neuron $j$. 
By learning the low-dimensional states ${\bf A}[i], {\bf B}[j]$ rather than each ${\bf W}_{ij}$ individually, we drastically reduce the number of free parameters and enable parameter sharing.
If two different weight matrices in different layers connect similar groups of neurons, those groups can be represented by the same or overlapping state vectors, and only the ``interaction pattern'' (a lightweight transformation) needs to differ. 
Indeed, recent work on the sharing of bases between layers in transformers shows that using a common set of basis vectors for the decomposed weights of multiple layers produces better compression than treating each layer independently~\citep{wangbasis}. 
These results support the feasibility of 
\emph{shifting our optimization target from massive individual weight values to a much smaller set of neuronal states}.

Adopting this NGC reparameterization not only compresses the model but also provides a new dynamical systems perspective on neural computation.
By interpreting the forward pass as a time-variant neuronal state evolution, we can apply tools from dynamical systems theory to analyze and improve the network's behavior.
In particular, we derive a metric called neuronal stability, drawing on Lyapunov’s stability criterion and input-to-state stability concepts, to quantify the robustness of internal neuronal dynamics.
Intuitively, the neuronal stability score measures how strongly the activations of neurons tend to contract toward stable patterns as the model processes each token or step of input.
A better stability score indicates that small perturbations or deviations in neuron activations die out over the course of processing, leading to smooth and consistent internal trajectories. 
By viewing neuron interactions through this lens of time-varying state evolution, we can formulate training objectives that encourage stable, coherent trajectories of activations during inference. 
This approach offers a novel form of regularization and fine-tuning, distinct from conventional weight decay or other layer-wise regularization techniques, by directly shaping the dynamical behavior of the neural state space.


Crucially, our analysis uncovers an intriguing link between stability and the network's reasoning capabilities. 
One might expect that making a system more stable (contractive in its dynamics) would render its behavior more predictable. 
However, we find the opposite in certain regimes: 
better neuronal stability often coincides with less predictable global behavior in the NGC-based model, and this correlates with improved reasoning performance on challenging tasks.
In essence, when local neural interactions are stabilized and made more efficient, the model as a whole can afford to explore more non-linear, unexpected trajectories in its activation space. 
We interpret this finding as evidence that emergent reasoning arises from a delicate interplay between stability and a controlled departure from linearity. 
In our framework, an external ``push'' or potential can drive the neural system away from its routine, allowing it to escape facile or biased responses, while the underlying stability ensures these excursions do not lead to chaotic divergence. 
This view suggests that a capacity for reasoning may correspond to the network being held in a poised state: 
stable enough to be reliable, yet with an external drive that injects the surprise or creativity needed to go beyond routine patterns.

While NGC is a general framework applicable to any large neural network, in this paper, we focus on Large Language Models (LLMs) as an important instantiation.
LLMs represent some of the most extreme cases of model scale and exhibit prominent emergent behaviors such as few-shot reasoning and step-by-step problem solving, making them an ideal testbed for our theory.
We empirically evaluate NGC on several open-source LLMs across challenging reasoning benchmarks. 
Under moderate compression ratios, NGC consistently outperforms standard low-rank approximation baselines and recent basis-sharing methods, demonstrating that we can shrink model size while enhancing performance on complex tasks. 
These results illustrate the practical merit of re-organizing neural computation around dynamic groups.
Moreover, on the theoretical side, we show that any standard neural network can be re-expressed as an NGC system composed only of intra- and inter-group communications (establishing a form of universality), and we prove that a model's reasoning ability under an NGC policy is directly connected to the dynamical stability of its information flow. 
Together, our findings indicate that NGC offers a new principled avenue for designing AI systems that are efficient in computation and rich in emergent functionality. 
We conclude by discussing open questions and future directions, including how the dynamics of neuronal groups might relate to generalization (e.g. as minimal surfaces in state-space) and how NGC principles could be applied beyond language to other domains.

To better introduce the NGC policy, we need to clarify some taxonomy on the NGC-based communication:
\begin{itemize}
    \item The communication \emph{type} categorizes the neuronal communications into \emph{intra-group} interactions among neurons within one neuronal group and \emph{inter-group} interactions among neuronal groups;
    \item The communication \emph{mechanism} or \emph{metric} performs the neuronal interaction for each type of communication. 
    For intra-group, this includes dot products or compact MLPs; for inter-group, this includes synchronous, instantaneous coupling, delayed, or asynchronous information exchange;
    \item The communication \emph{topology} depicts the directed graph of the neuronal groups, indicating which neuronal groups talk to whom and with what communication mechanism.
\end{itemize}
Then, an NGC policy is defined as a communication topology that applies different communication mechanisms in different types to the neural blocks of a neural structure such as LLM, forming a dynamical neural system.
In particular, we will propose and validate the following five claims that structure the rest of the paper: 
\begin{enumerate}
    \item \textbf{Universality}: 
    Neural computation in any well-formed neural structure is equivalent to an NGC process consisting only of intra-group and inter-group communications 
    (see Section~\ref{subsec: NGC-policy, intra/inter-group Com} and Theorem~\ref{app-thm: intra/inter-group coms are sufficient});
    \item \textbf{Topology-reasoning correlation}: 
    Reasoning ability of a neural system under an NGC policy correlates with dynamical stability of the induced token-time information flow 
    (see Section~\ref{subsec: dy-Neural system and neuronal stability} and Theorem~\ref{app-thm: neuronal stability measures reasoning capacity}); 
    \item \textbf{Realizability}: 
    Any inter-group communication, including delayed or asynchronous ones, is equivalent to instantaneous sharing of input-side neuronal states 
    (see Section~\ref{subsec: initialization and training} and Theorem~\ref{app-thm: shared input neurons simulate any NGC});
    \item \textbf{Emergent reasoning}: 
    Emergent reasoning correlates with an external potential that pushes trajectories away from routine linear behavior while preserving stability 
    (see Section~\ref{subsec-discuss: emergent reasoning} and Theorem~\ref{app-thm: emergent potential equals non-predictability}).
    \item \textbf{Metric symmetry}: 
    A well-formed intra-group communication mechanism that improves parameter efficiency is a bilinear symmetric Riemannian metric (see Section~\ref{subsec-exp: interact-Metric and ablation} and Proposition~\ref{app-prop: intra-group metric is riemannian}); 
\end{enumerate}

\begin{figure*}[h]
\centering
    \includegraphics[width=6.5in]{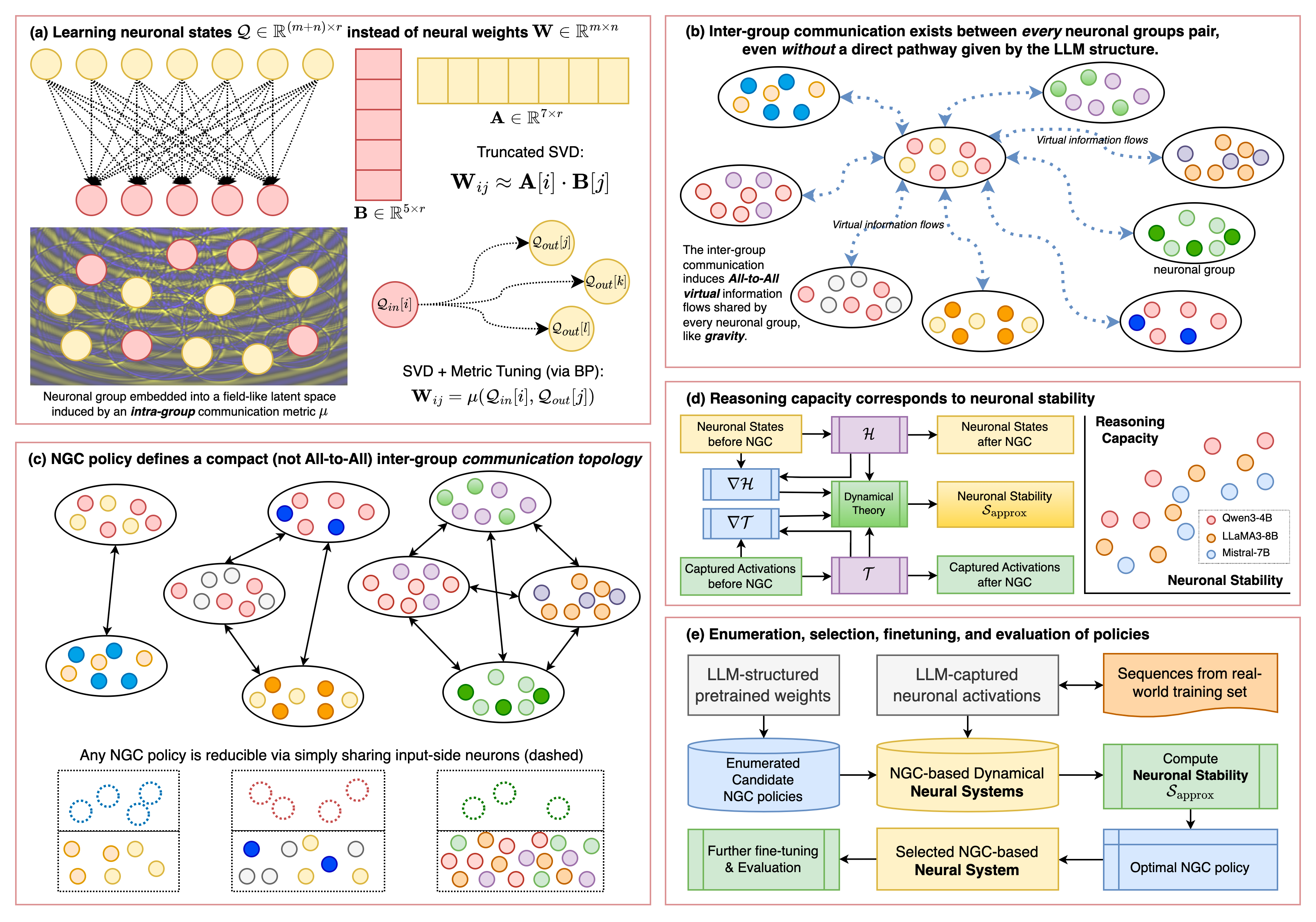}%
    \caption{
\textbf{An overview of the Neuronal Group Communication (NGC) pipeline.}
\textbf{(a) Neural weights as intra-group dynamical interaction.} 
We reinterpret neural weights as \emph{intra-group communication}, which are compact bilinear maps (optionally metric-symmetric) that define local dynamics over neuronal states, casting the model as a dynamical system (Section~\ref{subsec: learning neurons not weights}).
\textbf{(b) Virtual inter-group communication exists everywhere.} 
Groups exchange \emph{virtual information} through the alignment of their state subspaces; 
even without a direct pathway, a low-dimensional \emph{communication subspace} exists. 
(Section~\ref{subsec: NGC-policy, intra/inter-group Com})
\textbf{(c) NGC policy controls communication.} 
An NGC policy determines which neuronal groups share synchronous/asynchronous information exchange (Section~\ref{subsec: dy-Neural system and neuronal stability}); 
Any synchronous/asynchronous inter-group NGC among neuronal groups is equivalent to partially aligning their input-side neuronal states (Theorem~\ref{app-thm: shared input neurons simulate any NGC}).
\textbf{(d) Reasoning corresponds to neuronal stability.} 
Better reasoning is associated with reduced Lyapunov-like scores $\mathcal{S}$ (and surrogate $\mathcal{S}_{\text{approx}}$) under contracting dynamics (Theorem~\ref{app-thm: neuronal stability measures reasoning capacity}); 
external potentials $\nabla U$ produce controlled departures that remain stable from input-to-state (Theorems~\ref{app-thm: emergent potential equals non-predictability}).
\textbf{(e) Enumeration, selection, finetuning, and evaluation of policies.} 
We instantiate a bank of NGC policies from pre-captured activations, select by neuronal stability (low $S$ or $S_{\text{approx}}$), then finetune and evaluate on reasoning benchmarks (Section~\ref{subsec: initialization and training}).
    }
    \label{fig: overall-pipeline}
\end{figure*}



\clearpage
\section{Background and Motivation}

\subsection{Neural Model Compression and Low-Rank Representations}

Compression of large neural networks has a rich history, with approaches ranging from weight pruning and quantization to knowledge distillation and tensor factorization. 
For LLMs, post-training compression (compressing a pre-trained model without retraining from scratch) is particularly attractive. 
Among these methods, low-rank approximation has emerged as a powerful technique, given that transformer weight matrices often have rapidly decaying spectra (many small singular values). 
\citet{wang2024svd} pioneered a truncation-aware SVD pipeline for LLMs, introducing data-driven criteria to select which singular values to keep and fine-tuning compressed weights to compensate for lost information. 
It outperforms earlier layerwise SVD truncation methods, especially at high compression ratios (aggressively low rank). 
Based on this, SVD-LLM V2~\citep{wang2025svd} and Dobi-SVD~\citep{qinsidobi} incorporate activation awareness, recognizing that not all directions in weight space are equally important given the typical activations in the network~\citep{hao2024stabilized,huang2025sola}. 
For example, Activation-aware SVD (ASVD) scales weight rows/columns by the average activation magnitude in each channel before factorizing~\citep{yuan2023asvd}, which better preserves important features and yields additional memory savings (including compressing the attention KV cache during inference). 
These methods illustrate a key point: 
weight compression can benefit from incorporating inference dynamics (how activations use the weights) rather than purely static matrix approximation. 

Our NGC framework takes this philosophy further by explicitly tying compression to a network dynamics model. 
Instead of compressing each weight in isolation, we restructure the network into neuronal states that can be efficiently updated and shared.
Moreover, parameter sharing across layers is another promising avenue to eliminate redundancy. 
This idea was exemplified in Basis Sharing~\citep{wangbasis}, which factorizes the weight matrices of multiple layers in a common set of basis vectors. 
For instance, all attention projection matrices might share a basis in their low-rank decomposition, differing only in the combination coefficients per layer. 
By doing so, \citet{wangbasis} achieves lower perplexity (better performance) than independent SVD compression, under the same memory budget. 
This indicates that many features extracted by earlier layers are re-used in later layers, precisely the kind of redundancy NGC aims to capture through global neuronal state embeddings. 
In fact, if a set of neurons plays a similar role in multiple neural layers of neuronal groups, NGC would naturally represent them with one group state (or highly aligned states) that persists across those neural blocks, rather than duplicating parameters for them.

Beyond SVD, it is worth noting that knowledge distillation~\citep{hinton2015distilling} and pruning~\citep{liu2018rethinking} also implicitly assume some compressible structure in networks. 
In addition, studies of lottery ticket hypothesis~\cite{frankle2018lottery} have suggested that large networks contain smaller sub-networks that can be trained to full performance if isolated and reset~\citep{csordas2021neural}. 
These insights resonate with neuronal groups: 
a pruned sub-network could correspond to activating only certain neuron communities necessary for the tasks. 
However, naive pruning or compression can harm model stability and capacity, especially for LLMs where subtle interactions matter. 
NGC provides a structured alternative: 
we do ``pruning'' at the level of weight singular values (low-rank truncation) and ``distillation'' at the level of neuronal states, all within a unified dynamical system viewpoint.

In summary, model compression research indicates that many parameters in LLMs are wasteful or repeated, and that aligning compressions with network usage (activations, multi-layer reuse) is crucial. 
NGC is intrinsically aligned with this: 
by focusing on embedding-like neuronal states within a neuronal group (see Figure~\ref{fig: overall-pipeline}), we eliminate redundant degrees of freedom (since many weights correspond to the same underlying neuron interactions) and ensure that what is preserved (the neuronal states) are exactly those components that actively participate in representing knowledge.

\subsection{LLMs as Dynamical Systems: Representation and Stability}

The interpretation of recurrent neural networks as dynamical systems is well-established in deep learning theory~\citep{funahashi1993approximation,draye1996dynamic,chang2019antisymmetricrnn}, and even feed-forward networks can be analyzed in this light by considering the evolution of internal activations across layers or time steps~\citep{narendra1992neural,popescu2009multilayer,pei2023dynamics}. 
Transformers are not recurrent in the traditional sense, but when processing a sequence of tokens, token-wise activations can be viewed as a dynamical trajectory driven by the network’s layers~\citep{geneva2022transformers,fein2025flowing}.
Moreover, with approaches such as deep equilibrium models~\citep{bai2019deep} or sequence-to-sequence models~\citep{sutskever2014sequence}, 
researchers have drawn parallels between Transformers and iterative dynamical systems that converge to fixed points or cycle through states~\citep{lu2019understanding,hu2024state}. 
In our NGC framework, we explicitly construct a dynamical system by allowing neuron states to evolve over an artificial time parameter (distinct from the token index). 
This is analogous to introducing recurrence at the neuron level, superimposed on the forward pass of the network.

The reason we consider a dynamics-inspired perspective is because we need powerful analytical tools such as dynamical stability theory~\citep{bhatia2002stability,bhatia2006dynamical}. 
We hypothesize that \emph{an LLM’s capability for systematic reasoning and generalization is related to the stability of its internal state dynamics as it processes inputs}. 
Intuitively, if the effect of a new token on the hidden state decays or contracts in some metric ({\it i.e.}, small perturbations do not explode through the network), the model might exhibit more reliable multi-step reasoning, avoiding chaotic behavior. 
This intuition is supported by findings in sequence modeling: 
for example, certain state-space models (SSMs) were designed with provable stability (using diagonal plus low-rank state matrices) and have shown robust performance on long sequences~\citep{hamilton1994state,aoki2013state,gu2021efficiently,gu2023mamba}. 
In Transformers, recent work has drawn connections between the spectrum of weight matrices (or Jacobians of network mapping) and the tendency of the model to be stable or chaotic~\citep{arroyo2025vanishing}. 
Techniques like orthogonal initialization or spectral normalization can enforce a form of contractivity, which link to smoother optimization and better generalization in neural networks~\citep{csordas2021neural}.

In NGC, we formalize a notion of neuronal stability by deriving a Lyapunov-like function that aggregates the deviations between the ``root'' (original model without inter-group communication) and ``com'' (after-NGC model during inter-group communication). 
As presented in Eq.~\ref{eq: define-grad-T}, we track how much neuronal activations $\mathcal{A}(t)$ in the NGC model differ from those of the original model and how much neuronal states $\mathcal{Q}(\tau)$ change during communication. 
We also define residual terms for transitions and communications (see Eq.~\ref{eq: define-act-residual}-~\ref{eq: define-weight-residual}), 
and then a combined stability score $\mathcal{S}$ that essentially penalizes large residuals or expansive non-contracting mappings (see Eq.~\ref{eq: define-neuronal-stability-score}). 
We show an important result: 
under mild conditions, a lower stability score (meaning the NGC dynamics is more contractive and closely shadows the original model’s step-by-step behavior) correlates with higher LLM reasoning accuracy, especially on tasks that require multi-step deduction. 
This aligns with the notion of Input-to-State Stability (ISS) in control theory, where a system is stable if bounded input perturbations cause only bounded (ideally diminishing) deviations in state~\citep{sontag1995characterizations,jiang2001input}. 
In the  context of LLMs, the ``input'' is each new token, and the ``state'' is the hidden representation; 
a stable system would process each token in a way that does not derail the overall encoding of the sequence. 
We will later show empirical evidence: 
configurations of NGC that yield higher stability scores (less stable) tend to perform worse on long-context tasks, 
whereas more stable configurations (achieved by appropriate grouping and state-sharing) yield better performance.

Another motivation inspired by dynamics is to incorporate real-time activation information into the model representation. 
Traditional pruning or low-rank compression operates on weights offline. 
However, the actual weight utilization of a network during inference may vary depending on the inference sample. 
In NGC, by formulating communication states that depend on both weights and activations (see Eq.~\ref{eq: C_ij via SVD}), we essentially perform a form of online low-rank adaptation: 
we identify which directions in weight space are truly used by the given input distribution (through observed activations) and prioritize those in constructing the neuronal states. 
This is analogous to continuously learning a reduced model of the Jacobian or Hessian network. 
Recent work~\citep{chen2021drone,yuan2023asvd, wang2024svd, lin2024awq, wang2025svd} attempted ``activation clustering'' or joint weight-activation factorization to compress models, but those often treated each layer separately and could not capture global dynamics. 
Our approach, on the contrary, ties all layers together via the shared neuronal state space and the global dynamical update rules.

In summary, treating LLM computations as a dynamical system of neurons provides theoretical tools, {\it e.g.}, Lyapunov stability analysis, and practical mechanisms, {\it e.g.}, activation-informed learning, to interpret neural models. 
It moves us beyond viewing compression as a purely numerical approximation problem and toward viewing it as designing an optimal dynamical system that emulates the original complex one.

\subsection{Neuronal Groups and Modularity in Neural Networks}
The concept of grouping neurons into functional modules has deep roots in both artificial neural networks and neuroscience. 
In the brain, neuronal assemblies or cortical columns can be seen as groups of neurons that strongly interact among themselves and carry out specific functions or represent particular concepts~\citep{gerstein1989neuronal,mountcastle1997columnar,shipp2007structure}. 
Such grouping is believed to confer robustness and efficiency, each group can compute relatively independently and communicate succinct high-level information to other groups~\citep{pei2024modeling}.
In artificial networks, Modular Neural Networks have long been studied as a way to improve the learning of multiple skills or reduce interference between tasks~\citep{auda1999modular,qiu2024unlocking}. 
A classic example is the Mixture-of-Experts (MoE) architecture~\citep{jacobs1991adaptive,fedus2022switch}, where a gating mechanism activates one or a few expert sub-networks, {\it i.e.}, a neuronal group, for a given input. 
MoEs in large-scale transformers have achieved state-of-the-art results in both language and vision by expanding the model capacity while keeping the computation per token constant~\citep{shen2023moduleformer}.

However, most LLMs today are trained as ``monoliths'', without an explicit modular structure. 
As discussed earlier, evidence of implicit modularity has been found.
\citet{li2022large} shows that in a pretrained T5 model~\citep{raffel2020exploring}, only 3 to 6\% of neurons fire per input, 
and those that fire can be clustered by the type of stimulus (task or semantic feature) that activates them. 
Another line of work observes that certain neurons in LLMs correspond to specific concepts, and sets of neurons can be jointly linked to high-level functionalities~\citep{suau2020finding, dai2022knowledge}. 
When these neurons are ablated or their activations are perturbed, only specific capabilities are lost, implying that they constitute a semi-independent module~\citep{zhang2023emergent}. 
There is also research showing that during training, networks first learn coarse, high-level partitions of functionality (rough modular structure) and then refine the details~\citep{pfeiffer2023modular}. 
In other words, modularity emerges early in training, suggesting that it is a natural solution to organize complex tasks.

Our NGC approach explicitly embraces this principle of neuronal modularity. 
We define a neuronal group (or community) as a subset of neurons with dense internal connections but sparse external connections. 
This definition can be operationalized by analyzing the weight matrix: 
a perfect modular structure would correspond to a block-diagonal weight matrix (each block being internal connections) with few off-diagonal entries (external connections). 
Real networks are not so neatly blocked, but low-rank structure in weights can be seen as a relaxation: 
if weight ${\bf W}$ is low-rank $={\bf A}{\bf B}^{\top}$, then each neuron in the ``output'' of ${\bf W}$ is effectively connected to the ``input'' neurons through a low-dimensional subspace (spanned by columns of ${\bf B}$), which is a kind of bottleneck or interface. 
Neurons sharing similar ${\bf B}$ projections (or similar ${\bf A}$ rows) are effectively part of a community with shared communication channels identified by performing SVD on weights. 
Moreover, by allowing different weight matrices to share neuronal states, we merge the neuronal groups, which are artificially separate due to structural boundaries but in truth serve a similar role. 
It reveals that many neurons across different layers might be doing duplicate work because the architecture does not allow them to be unified. 
Our Proposition~\ref{app-prop: overlapping neurons} formalizes that if two sets of neurons in different parts of the network have highly overlapping communication patterns, then aligning them into one group does not affect the model’s computations.

Based on this insight on neuronal groups, we then turn into a framework where each neuronal group can conduct inter-group communication with any other neuronal groups, even there is no physical pathway between them.
For instance, in a given LLM, though a neural block ``q-projection'' in a neural layer $\ell$ has no direct pathway with a neural block ``k-projection'' in another neural layer $\ell+10$, they can still maintain a ``weak'' but solid inter-group communication that share ``virtual'' information flow between them. 
As in systems neuroscience, analysts distinguish structural connectivity (physical wiring) from functional effective connectivity~\citep{friston2011functional}. 
The inter-group communication aligns with the latter: 
an inferred interaction that need not coincide with a direct pathway.
Empirically, inter-area communication often concentrates in a low-dimensional communication subspace, 
exactly the ``few channels'' intuition we propose; 
this motivates defining flow via subspace alignment rather than literal message passing~\citep{semedo2019cortical}.
Moreover, using the geometry of representations to index ``who can talk to whom'' echoes representational similarity analysis and related model, comparison frameworks that treat proximity in representation space as functionally meaningful~\citep{kriegeskorte2008representational}.
Interpreting the learned metric on neuronal states through information geometry provides a principled Riemannian view of these proximities and geodesic notions~\citep{amari2016information,pei2024data}. 
The idea that proximity induces the notion of diffusion/flow in the manifold is consistent with diffusion geometry, a useful tool to discuss geodesic closeness and flow fields in state space~\citep{coifman2006diffusion}.

It is also illustrative to draw an analogy: \emph{Neuronal groups are societies}. 
In a human organization, small teams specialize in certain tasks and communicate with other teams via well-defined interfaces. 
An effective organization minimizes unnecessary communication (which creates noise and overhead) and delegates decisions within teams for their specialty areas. 
Likewise, NGC treats each neuronal group as a mini-agent that processes part of the signal and then communicates succinctly (through the low-dimensional state) to other neuronal groups. 
The neural weight matrix is analogous to a complete communication log, it records what each neuron would say to every other neuron if it had to. 
To avoid redundancy, the NGC tries to uncover the protocol of communication in a compressed form through shared representation. 
This view resonates with the best-effort communication principle of the Internet: 
do not enforce heavy, nonflexible communication constraints; 
allow each endpoint (neuronal group) to adapt and ensure the important information gets through efficiently. 

In summary, neuronal modularity is a real phenomenon in large networks, and exploiting it can yield efficiency gains. 
NGC is designed to provide a systematic way to surface and leverage this modularity by restructuring models around neuron groups and their communications. 
In the next section, we formally describe the NGC framework, including how we derive the low-rank neuronal states, how intra-group and inter-group communications are modeled, and how we train and evaluate the resulting neural system.

\clearpage
\section{The NGC framework and Methodology}

\subsection{Learning Neuronal States Instead of Neural Weights}
\label{subsec: learning neurons not weights}

In NGC, the fundamental trainable entities are the embedding-like neuronal states rather than individual weight parameters. 
We obtain these state vectors by factoring the low-rank weight matrices of the original model. 
Concretely, consider a linear mapping $\mathbf{y} = {\bf W} \mathbf{x}$, where ${\bf W}\in \mathbb{R}^{m\times n}$ is the neural weights built from $m+n$ neurons, $\mathbf{x}\in \mathbb{R}^n$ is the input activation vector (from $n$ input neurons) and $\mathbf{y}\in \mathbb{R}^m$ is the output activation (for $m$ output neurons). 
We can divide the neurons involved into two groups: 
$G_{\text{in}}$ for the $n$ input neurons and $G_{\text{out}}$ for the $m$ output neurons. 
The weight matrix ${\bf W}$ defines the communication from each neuron $j\in G_{\text{in}}$ to each neuron $i\in G_{\text{out}}$. 
Now suppose ${\bf W}$ has rank approximately $r$ (often much smaller than $m$ or $n$). 
Then there exist matrices ${\bf A}\in\mathbb{R}^{m\times r}$ and ${\bf B}\in\mathbb{R}^{n\times r}$ such that ${\bf W} \approx {\bf A}{\bf B}^{\top}$. 
We interpret ${\bf A}$ as containing $r$-dimensional state vectors for the output neurons (each row ${\bf A}[i]$ is the neuronal state of the neuron $i\in G_{\text{out}}$), and ${\bf B}$ as containing $r$-dimensional state vectors for the input neurons (each row ${\bf B}[j]$ is the neuronal state of the neuron $j\in G_{\text{in}}$). 
The weight interaction is then ${\bf W}_{ij} \approx {\bf A}[i]\cdot {\bf B}[j]^{\top}$, a dot product between the state of neuron $i$’ and the state of neuron $j$’. 
In other words, the communication strength between two neurons is the inner product of their state embeddings in a $r$-dimensional latent space.
Moreover, if we use a more complex trainable nonlinear metric $\mu:
\mathbb{R}^{r}\times\mathbb{R}^{r}\mapsto\mathbb{R}$ than the dot product, we might obtain a better approximation of ${\bf W}$ than ${\bf A}{\bf B}^{\top}$ ~\citep{pei2023dynamics,shen2024expanding,pei2024data}.

This reinterpretation has many implications. 
First, it provides natural compression: 
storing ${\bf A}$ and ${\bf B}$, which have $(m+n)r$ parameters, is much cheaper than storing ${\bf W}$, which has $mn$ parameters, when $r \ll \min(m,n)$. 
All modern LLM compression results indicate that moderate to high compression is indeed possible without a severe performance drop, confirming that the effective $r$ for many weight matrices is quite small. 
Second, it shifts our perspective to the neurons themselves as carriers of information. 
Each neuron $i$ is associated with a point ${\bf A}[i]$ in $\mathbb{R}^r$ that represents its general ``position'' or role in the communication space of the network, regardless of how many other neurons it connects. 
Training the model then largely amounts to learning these points in space, such that the dot products ${\bf A}[i]\cdot {\bf B}[j]^{\top}$ produce the desired computations. 
Third, it opens the door to sharing and reusing neuronal states. 
In a standard neural model, neural blocks that involve the same neuron (or a conceptually similar neuron) would be two separate sets of parameters.
As in multi-head attention: 
each head has its own projection matrices ${\bf W}_Q, {\bf W}_K, {\bf W}_V$.
Under NGC, if two heads are essentially detecting the same query-key relationship but in parallel, we could represent that by sharing some of the query or key neuronal states across those heads. 
This reduces parameters and might even improve generalization by preventing heads from diverging into duplicate roles.

In practice, we can simply perform a truncated SVD (or any low-rank factorization technique) on the major weight matrices of a pretrained LLM (attention Q, K, V projections, output projections, MLP weights). 
The choice of rank $r$ is dictated by the desired compression ratio. 
We initialize ${\bf A}$ and ${\bf B}$ from this factorization. 
However, many LLM weights are very large, so computing an exact SVD can be costly. 
Instead, one can use random projection or incremental SVD algorithms to find a basis that captures, say, 90–95\% of the variance in ${\bf W}$. 
Another detail is that in transformers, weight matrices often come in pairs ({\it e.g.}, ${\bf W}_Q$ and ${\bf W}_K$ map from the same input dimension to possibly different output dims). 
We may initialize each separately, but NGC will later allow aligning their subspaces if beneficial.

One might ask: 
do we lose accuracy by doing this factorization? 
By itself, it is a compression. 
But NGC does not stop at factorizing; 
it treats these factors as trainable neuron states that will be further optimized (with a small amount of data or even during normal model fine-tuning). 
We also do not insist that the initial factorization fully reconstructs ${\bf W}$, but only that it is a good starting point. 
In fact, we use a two-phase process: 
state initialization and state training, which is necessary when using trainable metric between neurons beyond the simple dot product. 
Empirically, we found that even a random low-rank initialization of ${\bf A},~{\bf B}$ can work if followed by training, but the use of SVD initialization speeds up convergence and ensures that we start in a reasonable basin.

\subsection{NGC Policy: Intra- and Inter-Group Communication}
\label{subsec: NGC-policy, intra/inter-group Com}
Within a single neuronal group, {\it e.g.}, all neurons in an MLP layer, the low-rank factorization above captures the intra-group communication as the neurons in that group interact via shared states. 
For example, in a feed-forward layer, the activation of each output neuron is a weighted sum of the activations of the input neurons; 
after factorization, this becomes a sum over latent channels $r$, where each channel mixes a certain combination of input neurons with a combination of output neurons. 
We can think of those $r$ channels as $r$ ``topics'' or ``sub-communications'' that occur within the group. 
This is analogous to how within a company department, information might flow along a few key channels (like memos for specific purposes) rather than every individual separately talking to each other.

\paragraph{Remark (Virtual information flow).}
In this paper, \emph{information flow} refers to a \emph{notion of virtual} representation-level, not literal transmission of raw signals such as $y=Wx$ between neural blocks. 
There are two reasons. 
First, any explicit exchange of signals is already encoded in neuronal activations, so tracking those signals adds little explanatory value at the level of representation. 
Second, the true routing graph of a large language model spans many layers and residual paths and is therefore hard to analyze directly. 
We instead study a \emph{virtual flow} induced by the geometry of \emph{neuronal states}: 
when the input-side or output-side subspaces of two groups $G_i$ and $G_j$ are closely aligned under the learned metric, they effectively share a low-dimensional \emph{communication subspace} through which task-relevant activity could, in principle, be routed. 
Such a virtual flow can exist even without a physical pathway ({\it e.g.}, between a layer-1 $q$-projection and a layer-20 $k$-projection). 
Intuitively, it quantifies the \emph{counterfactual capacity to exchange information}: 
how strongly $G_i$ would influence $G_j$ if a light inter-group link was enabled. 
Formally, it is captured by spatial proximity and alignment in the state manifold ({\it e.g.}, similarity of representational patterns and geodesic closeness under the group-wise metric) and ties to our dynamics: 
a strong virtual flow predicts reduced residual energy when the corresponding link is instantiated.

Beyond intra-group communication, NGC also proposes inter-group communication, {\it i.e.,} allowing neurons in different groups to communicate and share information. 
In a standard feed-forward network, different layers are sequential and interact only through the activations passed between them; 
there is no direct weight connecting a neuron in the layer $\ell$ to a neuron in the layer $\ell+2$. 
However, consider that a neuron in layer $\ell$ and a neuron in layer $\ell+2$ might represent very similar concepts (due to how the model trained). 
If we allow them to share the same neuronal states in NGC, effectively we ``merge'' those two neurons into one underlying neuron (from the state space perspective), even though they live in different layers in the original model. 
This is one form of inter-group communication: 
neurons across layers (or across submodules like attention and feed-forward) can be synchronized to share neuronal states. 

Concretely, if ${\bf q}_i$ is the neuronal state of the neuron $i$ in one neuronal group and ${\bf q}_k$ is the neuronal state of the neuron $k$ in another group, and if we set ${\bf q}_k = {\bf q}_i$, then the information that the neuron $i$ handles is now immediately available to the neuron $k$.
They are effectively the same underlying neuron that participates in two different parts of the network. 
Original connectivity, {\it i.e.,} neural weights, is adjusted accordingly so that wherever $k$ is used, we also use the neuronal state of $i$ or vise versa. 
We do this in a way that preserves the originally learned functions as much as possible: 
typically, we would choose to merge neurons that we suspect had redundant roles, as indicated by similar weight patterns or similar activation behaviors.

More generally, rather than outright merging, NGC can allow partial sharing via aligning neuron states. 
In practice, we implement inter-group communication by constructing certain neuron groups that span what were separate communities. 
For example, suppose that layer $\ell$ and layer $\ell+1$ both have key-projection weight matrices ${\bf W}_K^\ell$ and ${\bf W}_K^{\ell+1}$. 
We might choose a structure called ``kk-qv'' where the \textbf{key} neurons of these two layers are considered as a neuronal group with inter-group communication, which means that we use a joint state matrix $\mathcal{Q}_{\text{key}}$ for both key projections of both layers. 
This means that the keys in layers $\ell$ and $\ell+1$ are constrained to lie in the same state subspace (or even exactly share neuronal states if fully merged). 
This inter-group tie can significantly reduce parameters and might not hurt accuracy much because those keys are likely doing similar things. 
Indeed, we will see that certain inter-layer communications improve stability and accuracy, while others do not, suggesting that some layers truly benefit from being unified whereas others must retain individuality.

For brevity, 
we call the original model ``root'' and the communicating model that allows inter-group communication ``com''.
To formalize, let $\mathcal{Q}^{(\text{root})}$ and $\mathcal{A}^{(\text{root})}$ denote the neuronal states and activations in the original model with only \emph{intra-group} communication. 
The activations are obtained by feeding several training sequences containing thousands of tokens to the original model and the after-NGC system, which has been initialized via the SVD method (see Eq.~\ref{eq: C_ij via SVD}).
Intra-group communication produces a low-rank $r$-dimensional representation of the original model.
Let $\mathcal{Q}^{(\text{com})}$ and $\mathcal{A}^{(\text{com})}$ be the counterpart when \emph{inter-group} communication is also enabled.
Note that such a ``root'' system differs from the ``com'' system induced by the NGC policies such as ``q-k-v'' which do not require inter-group communication.
Although these NGC policies produce $\mathcal{Q}^{(\text{com})}$ numerically equal to $\mathcal{Q}^{(\text{root})}$, their activations are different, {\it i.e.}, $\mathcal{A}^{(\text{com})}\neq\mathcal{A}^{(\text{com})}$.
Intuitively, the ``root'' system is an ideally low-rank neural system whose neuronal activations are identical to the original model.

Now, we are concerned about how the reasoning capacity of the ``com'' system varies compared to the ``root'' system.
In other words, we expect to determine whether the reasoning capacity of an LLM, which corresponds to a specific NGC policy, improves or degrades based on the variation between the neuronal states of their neurons.
The change in reasoning capacity is principally due to the variation of the state dimension.
Generally, for a fixed compression ratio in parameter counts, 
if we allow several neuronal groups to communicate, we will effectively increase the dimensionality of their neuronal state, since they are provided with additional dimensions to accommodate shared information.
Let $r^*$ denote the enhanced neuronal state dimension when inter-group communication is allowed, typically $r^* > r$.
In practice, we adjust $r^*$ in advance to keep the total parameters within budget, {\it e.g.}, if merging groups saves parameters, we could spend a bit on increasing rank.

For better implementation, we enumerate the NGC policy including inter-group communication as an assignment of which neuronal groups share information. 
The simplest NGC policy is ``no inter-group communication'', 
in which each layer is factorized independently; 
this is similar to SVD-LLM. 
Another NGC policy could be ``all successive two layers share identical type of neuronal block'', {\it i.e.,} each ${\bf W}_{Q}^{\ell}$ merges with ${\bf W}_{Q}^{\ell+1}$, each ${\bf W}_{K}^{\ell}$ merges with ${\bf W}_{K}^{\ell+1}$, and each ${\bf W}_{V}^{\ell}$ merges with ${\bf W}_{V}^{\ell+1}$.
This policy is called ``qq-kk-vv'' and aligns with the optimal basis sharing mechanism in \citet{wangbasis}.
We explore many more structured NGC policies: 
sharing across different types of neural blocks, {\it e.g.,} ${\bf W}_{Q}^{\ell}$ merges with ${\bf W}_{V}^{\ell+1}$ or ${\bf W}_{K}^{\ell}$ merges with both ${\bf W}_{Q}^{\ell}$ and ${\bf W}_{V}^{\ell}$, {\it etc}.
Each NGC structure yields a different parameter efficiency and a different network behavior.
Next, we will develop a theoretical framework to measure and even predict the reasoning capacity of a neural system, corresponding to an LLM assigned with an NGC policy, without running the LLM on a validation set.

\subsection{Dynamics-Inspired Neural System and Neuronal Stability}
\label{subsec: dy-Neural system and neuronal stability}
Having established the structural aspect (neuronal states and their sharing), we now introduce the dynamical update rules that govern how neuron activations and neuronal states evolve in NGC. 
We conceive of the forward pass of the LLM and the evolution of neuronal states during training or an NGC as a coupled dynamical process, 
described by two equations (Eq.~\ref{eq: raw-dynamical-sys}, which is expanded to Eq.~\ref{eq: simplified-dynamical-sys} under certain conditions). 
At a high level, we treat the token index $t$ as a discrete time-step for activations $\mathcal{A}^{(t)}$, 
and an artificial time-step $\tau$ for the evolution of neuronal states $\mathcal{Q}^{(\tau)}$ (which typically evolves during training iterations or neuronal communication).
In fact, when we use the ``root'' and ``com'' notation,
and suppose $\mathcal{Q}^{(\text{root})} = \mathcal{Q}^{(\tau_0)}$,
then the after-NGC neuronal states $\mathcal{Q}^{(\text{root})}$ are equivalent to $\mathcal{Q}^{(\tau_0+T_{com})}$ for a specific $T_{com}$.
Here, we take $\mathcal{Q}^{(\tau)}$ and $\mathcal{Q}^{(\text{com})}$ as two separate variables, 
since the intermediate states between $\mathcal{Q}^{(\text{root})}$ and $\mathcal{Q}^{(\text{com})}$ are difficult to enumerate and analyze in a limited computational budget.

Now, the dynamics-inspired differential equations for the neural system assigned with an NGC policy are given by:
\begin{equation}\label{eq: raw-dynamical-sys}
    \begin{gathered}
        \frac{d}{dt}\mathcal{A}^{(\text{com};~t)}={\bf D}[\mathcal{A}^{(\text{com};~t)};~\mathcal{A}^{(\text{root};~t)}]+{\bf E}[\mathcal{Q}^{(\text{com};~\tau)};~\mathcal{Q}^{(\text{root};~\tau)}],\\
        \frac{d}{d\tau}\mathcal{Q}^{(\text{com};~\tau)}={\bf F}[\mathcal{Q}^{(\text{com};~\tau)};~\mathcal{Q}^{(\text{root};~\tau)}]+{\bf G}[\mathcal{A}^{(\text{com};~t)};~\mathcal{A}^{(\text{root};~t)}].\\
    \end{gathered}
\end{equation}
where $\{{\bf D},~{\bf E},~{\bf F},~{\bf G}\}$ are abstract linear operators or projections that mix the information between the communicating ``com'' and the ``root'' neural systems. 
The change in after-NGC ``com'' activations is driven by a combination of how the activations of the ``root'' model would change and how far the ``com'' neuronal states induced by NGC have moved from the ``root'' states;
and conversely, the change in ``com'' neuronal states induced by NGC is driven by differences between ``com'' and ``root'' states and activations. 
Although this formulation may seem complex, it essentially couples the optimization of the neuronal state with the forward computation in a way that tries to keep the neuronal behavior, {\it i.e.,} neuronal activations, of the after-NGC ``com'' model close to the ``root'' model.

Under Proposition~\ref{app-prop: simplified-dynamical-sys}, we show that if the training data cover diverse cases sufficiently, the above continuous dynamics can be reduced to a discrete dynamics:
\begin{equation}\label{eq: simplified-dynamical-sys}
    \begin{gathered}
        \mathcal{A}_{}^{(\text{com};~t+1)}=\lambda\cdot\mathcal{A}_{}^{(\text{root};~t+1)}\mathcal{T}_{}^{(trs)}+(1-\lambda)\cdot\mathcal{A}_{}^{(\text{com};~t)}\mathcal{T}_{}^{(com)},\\
        \mathcal{Q}_{}^{(\text{com};~\tau+1)}=\lambda\cdot\mathcal{Q}_{}^{(\text{root};~\tau+1)}\mathcal{H}_{}^{(trs)}+(1-\lambda)\cdot\mathcal{Q}_{}^{(\text{com};~\tau)}\mathcal{H}_{}^{(com)}.
    \end{gathered}
\end{equation}
for some $\lambda\in(0,1)$, and certain projection matrices $\mathcal{T}_{in}^{(trs)},~\mathcal{T}_{in}^{(com)}\in\mathbb{R}^{N_{in}\times N_{in}}$;
$\mathcal{H}_{in}^{(trs)},\mathcal{H}_{out}^{(trs)}\in\mathbb{R}^{r\times r^*}$;
$\mathcal{H}_{in}^{(com)},\mathcal{H}_{out}^{(com)}\in\mathbb{R}^{r^*\times r^*}$;
$\mathcal{T}_{out}^{(trs)},~\mathcal{T}_{out}^{(com)}\in\mathbb{R}^{N_{out}\times N_{out}}$. 
Intuitively, 
at each token step, the activations $\mathcal{A}^{(\text{com})}$ are updated to be a weighted average of (a transform of) the new activations of the root model and the previous activations of the NGC model. 
Similarly, neuronal states $\mathcal{Q}^{(\text{com})}$ evolve by blending the newly observed root states with the previous NGC states. 

In effect, NGC tracks the trajectory of the ``root'' system, but with some momentum or persistence from its own past state. 
The matrices $\mathcal{T}^{(\text{trans})}, \mathcal{T}^{(\text{com})}$ (for input/output activations) and $\mathcal{H}^{(\text{trans})}, \mathcal{H}^{(\text{com})}$ (for neuronal states) are determined by the specifics of how communication is structured. 
For example, $\mathcal{T}^{(\text{trans})}$ might be a selection matrix that selects the activation of the root neuron to incorporate, while $\mathcal{T}^{(\text{com})}$ could be a contraction of the current activation into a compact latent space. 
Eq.~\ref{eq: simplified-dynamical-sys} defines a stable iterative process that the neuronal activations of the NGC model follow to mimic the ``root'' system.
The pipeline of Eq.~\ref{eq: simplified-dynamical-sys} is presented in Figure~\ref{fig: pipeline}.

\begin{figure*}[t]
\centering
    \includegraphics[width=5.0in]{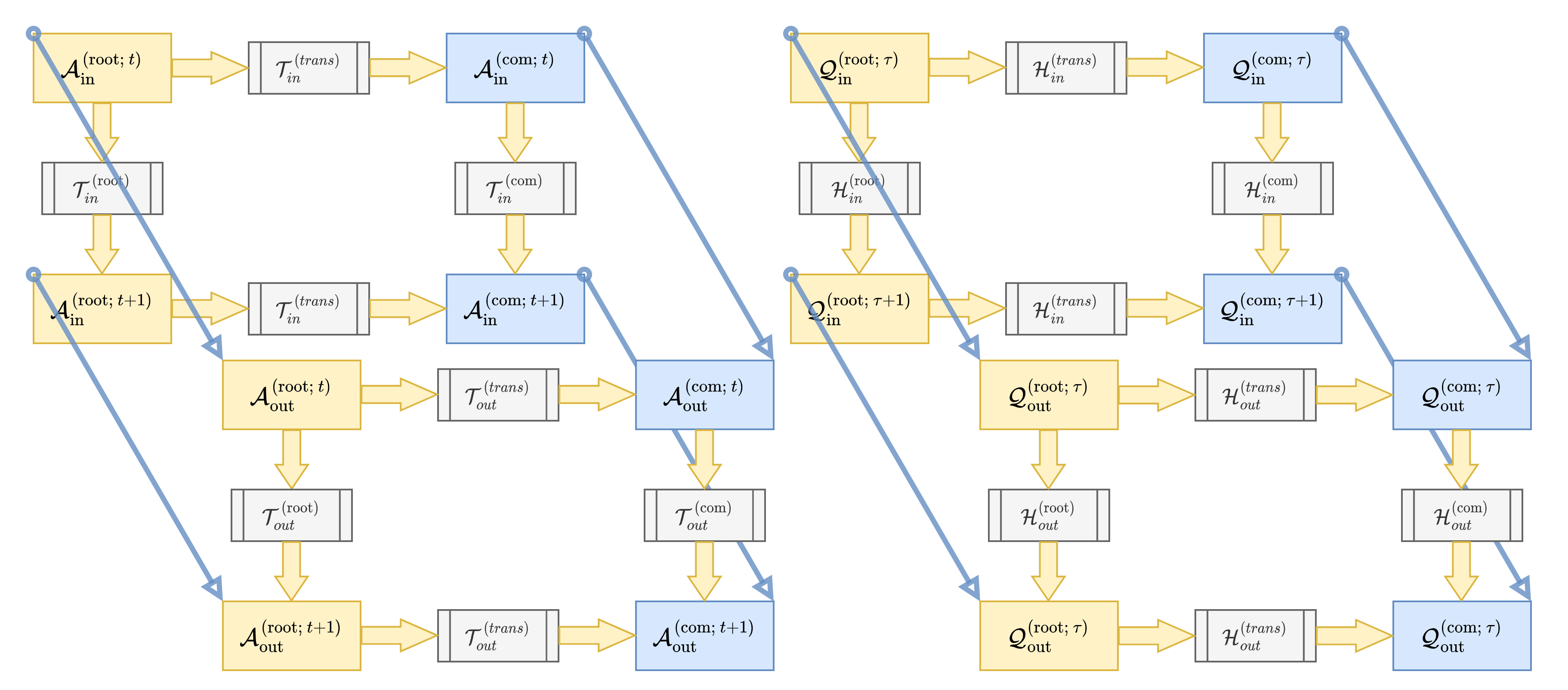}%
    \caption{
    The ``root'' system (original model) and the ``com'' system (under an NGC policy) evolve through coupled updates of activations $\mathcal{A}$ and neuronal states $\mathcal{Q}$. 
    Their token and structure-wise evolutions are approximated by linear projections $\mathcal{T}$ and $\mathcal{H}$, yielding a stable discrete update that lets the ``com'' system follow the ``root'' trajectory while permitting controlled deviations induced by inter-group communication.
    }
    \label{fig: pipeline}
\end{figure*}

Using Eq.~\ref{eq: simplified-dynamical-sys} as the backbone, we define a measure of the discrepancy, essentially, how far the after-NGC ``com'' system strays from the ``root'' system at each step. 
By Theorem~\ref{app-thm: neuronal stability measures reasoning capacity}, the defined neuronal stability score in Eq.~\ref{eq: simplified-dynamical-sys} mirrors the usual ISS (input-to-state stable)/Lyapunov bound and can measure the corresponding LLM reasoning capacity.
In other words: 
more contractive token dynamics and smaller residuals give better stability and, typically, better task accuracy.
To obtain the exact form, we need to define variational residuals for each neural block, {\it e.g.,} each layer or sub-layer, for both the transition part (``root'' vs. ``com'' activation) and the communication part.
\begin{equation}\label{eq: define-grad-T}
    \begin{gathered}
        \nabla\mathcal{T}_{x,\dagger}^{(trs)}=\frac{1}{T}\sum_{t}\frac{\mathcal{A}_{x,\dagger}^{(\text{root};~t)}\mathcal{T}_{x,\dagger}^{(trs)}-\mathcal{A}_{x,\dagger}^{(\text{com};~t)}}{\mathcal{A}_{x,\dagger}^{(\text{com};~t)}}\in\mathbb{R}^{N_{\dagger}}\\
        \nabla\mathcal{T}_{x,\dagger}^{(com)}=\frac{1}{T}\sum_{t}\frac{\mathcal{A}_{x,\dagger}^{(\text{com;~t})}\mathcal{T}_{x,\dagger}^{(com)}-\mathcal{A}_{x,\dagger}^{(\text{com};~t+1)}}{\mathcal{A}_{x,\dagger}^{(\text{com};~t+1)}}\in\mathbb{R}^{N_{\dagger}}
    \end{gathered}
\end{equation}
Here, $\dagger\in\{\text{in},\text{out}\}$ is the role of neurons.
Similarly, we define $\nabla\mathcal{H}_{x,\dagger}^{(trs)},\nabla\mathcal{H}_{x,\dagger}^{(com)}\in\mathbb{R}^{r^{*}}$.
These variational residuals also correlate with the nonlinearity by measuring the actual error induced by the best-fitted linear transformation.
For brevity, we define an expanded matrix form of the dynamical projections,
\begin{equation}\label{eq: define-augemented-Phi}
    \begin{gathered}
        \Phi_{x,\dagger} = 
        \begin{pmatrix}
            \mathcal{T}_{x,\dagger}^{(trs)} &
            \nabla\mathcal{T}_{x,\dagger}^{(trs)} &
            \mathcal{T}_{x,\dagger}^{(com)} &
            \nabla\mathcal{T}_{x,\dagger}^{(com)}\\
            \mathcal{H}_{x,\dagger}^{(trs)} & 
            \nabla\mathcal{H}_{x,\dagger}^{(trs)} &
            \mathcal{H}_{x,\dagger}^{(com)} &
            \nabla\mathcal{H}_{x,\dagger}^{(com)}
        \end{pmatrix}\in\mathbb{R}^{(N_{\dagger}+r^*)\times 4}
    \end{gathered}
\end{equation}
Thus, the neuronal stability score is computed as follows.
\begin{equation}\label{eq: define-neuronal-stability-score}
    \begin{gathered}
        \mathcal{S}=\sum_{x}\sum_{\dagger}\lVert\Phi_{x,\dagger}\mathcal{F}_{x}\Phi_{x,\dagger}^{\top}\rVert
    \end{gathered}
\end{equation}
where $\mathcal{F}_{x}\in\mathbb{R}^{4\times 4}$ is a learnable metric that acts as a Riemannian metric for each neural block.
In simpler terms, $\mathcal{S}$ is a weighted sum of the norms of these residuals; 
minimizing $\mathcal{S}$ encourages all residuals to be small, meaning that the step of NGC (transitional and communicative) matches closely the root’s. 
One might be concerned that this $\mathcal{S}$ does not clearly relate to the neural structure and the NGC topology. 
However, the spatial proximity of the neuronal states ({\it i.e.}, if inter-group communication has aligned neurons properly) already encodes the topology of how information should flow. 
Thus, simply summing the variational terms per-block is enough to capture a meaningful global property.
In practice, we derive an unsupervised approximation of $\mathcal{S}$ given by:
\begin{equation}\label{eq: define-neuronal-stability-score-simple}
    \begin{gathered}
        \mathcal{S}_{\text{approx}}=\sum_{x}\sum_{\dagger}\ln{\Big(\frac{\sigma_{max}(\mathcal{T}_{x,\dagger}^{(trs)})}{\sigma_{max}(\mathcal{T}_{x,\dagger}^{(com)})}\Big)}-\alpha\cdot\nabla\mathcal{T}_{x,\dagger}^{(com)}-\beta\cdot\nabla\mathcal{H}_{x,\dagger}^{(com)}
    \end{gathered}
\end{equation}
where $\sigma_{\max}(\mathcal{T})$ denotes the largest singular value of the matrix $\mathcal{T}$, 
and $\alpha, \beta$ are scalar hyper-parameters. 
Higher $\mathcal{T}_{x,\dagger}^{(com)}$ and $\mathcal{H}_{x,\dagger}^{(com)}$ correlate with the non-linearity in the time-varying dynamical ``com'' system.
Therefore, Eq.~\ref{eq: define-neuronal-stability-score-simple} also means that a better NGC policy often refers to a non-predictable conversion from ``root'' such that the best-fitted linear transformation is less capable of describing the token-varying evolution in the ``com'' system. 
Simply speaking, emergent reasoning capacity refers to the external forces that drive the neural system away from its usual dynamical behaviors.
We theoretically validate this argument in Theorem~\ref{app-thm: emergent potential equals non-predictability}.

We compute this $\mathcal{S}_{\text{approx}}$ cheaply for candidate NGC policies and use it to select the best policy prior to full training and inference. 
Normally, one would have to try all compression configurations and evaluate the validation data to choose the best. 
Here, we rather use a theory-guided metric to predict which configuration will be most stable (and presumably expressive). 
Indeed, we find that the structure with the best $\mathcal{S}_{\text{approx}}$ typically corresponds to the one with the best performance.

\subsection{Initialization and Training Protocol}
\label{subsec: initialization and training}

After choosing a communication structure (which defines which neurons/groups share states), we proceed to initialize and train the NGC model. 
The process is as follows:

\subsubsection*{Activation Capture}
We first run a collection of sequences (a few dozen up to a few hundred) through the original pretrained LLM (without any modification) and record neuronal activations $\mathcal{A}_{\text{in}}$ and $\mathcal{A}_{\text{out}}$ for the layers we plan to reinterpret. 
These sequences can be a subset of the training data or another representative text; 
the goal is to get a snapshot of typical activation patterns. 
This is used both to inform the initial SVD (if doing activation-weight combined factorization) and to set up any activation-aware transforms, {\it e.g.,} the diagonal scaling used in ASVD~\citep{yuan2023asvd}, though in our main pipeline we often do a simpler initialization.

\subsubsection*{State Initialization via Inverse Mapping}
Using captured neuronal activations and original weights, we apply an inverse communication mapping to solve for initial communication states ${\bf C}_{ij}$ between groups. 
In the  formula, for each pair of groups $G_i, G_j$ that are set to communicate in our structure, we compute
\begin{equation}\label{eq: define-Cij-via-inverseCom}
    \begin{gathered}
        {\bf C}_{ij}=\mathcal{I}_{ij}\big(\mathcal{A}_i~,\mathcal{A}_j~;{\bf W}_{ij}\big),~~~~~i,j\in\mathcal{N}.
    \end{gathered}
\end{equation}
where $\mathcal{I}_{ij}$ is a function that ``infers'' a low-rank communication state from the observed activations $\mathcal{A}_i, \mathcal{A}_j$ and the original weight connecting those groups.
Although Eq.~\ref{eq: define-Cij-via-inverseCom} is abstract, in practice we choose it to minimize a combination of activation reconstruction error and weight reconstruction error. 
Specifically, we define two residuals: 
the activation-restricted residual:
\begin{equation}\label{eq: define-act-residual}
    \begin{gathered}
        \mathcal{R}_{ij}^{(Act)}({\bf C}_{ij})=\sum_{t=1}^T\lVert {\bf a}_i^{(\tau)}{\bf W}_{ij}-{\bf a}_i^{(\tau)}{\bf C}_{ij}\rVert_2^2,
    \end{gathered}
\end{equation}
where ${\bf a}_i(\tau)$ is the column of $\mathcal{A}_i$ in the token step $\tau\in\{1,...,T\}$.
We then introduce the weight-restricted residual as follows.
\begin{equation}\label{eq: define-weight-residual}
    \begin{gathered}
        \mathcal{R}_{ij}^{(W)}({\bf C}_{ij})=\sum_{t=1}^T\lVert {\bf W}_{ij}-{\bf C}_{ij}\rVert_2^2,
    \end{gathered}
\end{equation}
We then set up a gradient flow:
\begin{equation}\label{eq: d_Cij/dt}
    \begin{gathered}
        \frac{d}{dt}{\bf C}_{ij}(\tau)=-\Psi\circ{\big[\mathcal{R}_{ij}^{(Act)}({\bf C}_{ij})~;~
        \mathcal{R}_{ij}^{(W)}({\bf C}_{ij})\big]}^{\top},
    \end{gathered}
\end{equation}
which means we iteratively adjust ${\bf C}_{ij}$ to reduce both residuals. 
In other  words, we want a communication state ${\bf C}_{ij}$ that both reproduces the effect of ${\bf W}_{ij}$ on the observed activations and is itself low-rank similar to ${\bf W}_{ij}$.
We show in Proposition~\ref{app-prop: C_ij via SVD} that the solution of this flow, at convergence, is equivalent to a formula:
\begin{equation}\label{eq: C_ij via SVD}
    \begin{gathered}
        {\bf C}_{ij}\leftarrow \lambda\cdot\mathcal{A}_i^{-1}\text{SVD}\big[\mathcal{A}_i{\bf W}_{ij}\big] + (1-\lambda)\cdot \mathcal{X}^{-1}\text{SVD}\big[\mathcal{X}{\bf W}_{ij}\big],~~~\lambda\in (0,1)
    \end{gathered}
\end{equation}
for some $\lambda\in(0,1)$. 
Here, $\mathcal{A}_i^{-1}$ is the pseudo-inverse of the activation matrix (which captures the subspace of ${\bf W}_{ij}$ actually utilized by the activations of the group $i$), 
and $\mathcal{X}$ is a random matrix of the same shape (injecting a small exploratory component to cover unused directions). 
Essentially, Eq.~\ref{eq: C_ij via SVD} says: 
take the SVD of the weighted activations $\mathcal{A}_i \textbf{W}_{ij}$, and also of a full weight sampling $\mathcal{X} \textbf{W}_{ij}$, and blend them. 

The result ${\bf C}_{ij}$ spans the principal subspace in which the information genuinely flows during inference, affecting the important parts of ${\bf W}_{ij}$ without the need for full rank. 
This step yields the initial neuronal states $\mathcal{Q}_{\text{in}}$ for the neuronal group $i$ and $\mathcal{Q}_{\text{out}}$ for the neuronal group $j$ by factorizing ${\bf C}_{ij}$ into $\mathcal{Q}_i \mathcal{Q}_j^{\top}$. 
In practice, an easier, but slightly less precise way to initialize is simply to perform SVD on ${\bf W}_{ij}$ itself but restricted to the top eigen-activations, as indicated by $\mathcal{A}_i$. 
We found that either method gives a good starting point for $\mathcal{Q}$.

\subsubsection*{State Alignment and NGC policy Optimization}

If our NGC policy requires that some neuronal groups are merged or share neurons, we perform a step of alignment. 
This means that if two groups $G_i$ and $G_j$ are supposed to communicate (share state), we identify which neurons overlap. 
Sometimes we allow partial overlap, and only a subset of neurons or a subset of the latent dimensions $r^*$ is shared. 
We then adjust the neuronal states $\mathcal{Q}$ of those neurons to be identical and adjust the remaining ones accordingly to not disrupt known communications ${\bf C}$. 
Simply speaking, an NGC policy tells us which neuronal groups should be merged, then we merge them by setting up a globally shared one.
If these to-be-shared neuronal groups have different shapes, then we use zero-padding to fit them.
Next, we apply Eq.~\ref{eq: C_ij via SVD} to initialize neuronal states under this NGC policy.

Theoretically and practically, we can simply use shared input neurons to simulate NGC policies (see Fig~\ref{fig: overall-pipeline}).
Our Theorem~\ref{app-thm: shared input neurons simulate any NGC} justifies that shared input neurons are expressive enough to represent arbitrary inter-group NGC, 
including slow/asynchronous communications; 
this instantaneous sharing realized by shared input neurons is a convenient engineering special case.
In practice, we are often provided with several NGC policies, each of which corresponds to an initialized neural system.
We can simply replace the neural blocks of an LLM with the neural system, {\it i.e.}, an operation like ${\bf y}_{j}\!=\!\sum_{i}{\bf W}_{ij}{\bf x}_i$ is implemented as ${\bf y}_{j}\!=\!\sum_{i}\mu(\mathcal{Q}_{i},\mathcal{Q}_{j})\cdot{\bf x}_i$, where $\mu$ is a metric such as dot-product operation.
Note that $\mu$ could also be a more complex trainable metric, {\it e.g.}, MLP, which requires an additional fast training to calibrate.
This fast training requires only 2-3 training sequences to generate neuronal activations for supervised learning.
After initialization, we capture neuronal activation using validation sequences, which consist of hundreds or thousands of tokens.
The captured activations account for $\mathcal{A}^{(com)}$ denoted in Eq.~\ref{eq: raw-dynamical-sys}.
Then, we can use Eq.~\ref{eq: define-neuronal-stability-score-simple} to measure the neuronal stability score of the current ``com'' neural system, and determine whether to keep this neural system as the final one for further fine-tuning and evaluation.

Once all initial states $\mathcal{Q}$ and interactive metrics ${\bf G}, {\bf H}$ between neuronal states are established if using a learned metric, 
we ``freeze'' the high-level architecture of the neural system, {\it i.e.}, no new merges or changes in grouping, and perform a brief fine-tuning of neuronal states in the training set. 
We do not update any original weights anymore; those have been replaced by our neuronal communications. 
Instead, we update the entries of $\mathcal{Q}_{\text{in}}$ and $\mathcal{Q}_{\text{out}}$ via the backpropagation algorithm~\citep{rumelhart1986learning}, using the task loss ({\bf, e.g.}, next-token prediction loss if language modeling, or a supervised objective if fine-tuning on a benchmark) plus possibly a stability regularization term. 
This training is very efficient: 
the number of trainable parameters is reduced ({\bf e.g.}, a 30\% compression means only 30\% as many parameters to update), and we can often achieve convergence in a small number of iterations. 
In our experiments, updating the neuronal states for an 8B model on a reasoning task takes only a few minutes on a single GPU. 
This is orders-of-magnitude less than full model fine-tuning or inference on the entire benchmark, highlighting the practicality of NGC as a post-processing step rather than fine-tuning.

\subsubsection*{Inference with NGC}
After training, we deploy the NGC system in test tasks. 
Inference proceeds as usual (forward passes through the transformer), except that at each layer, instead of a dense matrix multiplication, we do two smaller ones (${\bf A}[i] \cdot {\bf B}[j]^{\top}$ style) and if required, project through any shared metrics or state alignments. 
Specifically, for a typical neural operation ${\bf y}_{j}\!=\!\sum_{i}{\bf W}_{ij}{\bf x}_i$, 
we instead use ${\bf y}_{j}\!=\!\sum_{i}\mu(\mathcal{Q}_{i},\mathcal{Q}_{j})\cdot{\bf x}_i$ or ${\bf y}_{j}\!=\!\sum_{i}\mu(\mathcal{Q}_{i},\mathcal{Q}_{j};{\bf x}_i)$ in the NGC system, where $\mu$ is a pre-defined metric function such as the dot-product or some complex sub-networks.

We also implement the inter-group communications: 
for simultaneous sharing, as we assumed, if two groups share input states, we ensure that they use the same $\mathcal{Q}_{\text{in}}$ matrix. 
We interpret this simultaneous sharing as the input neurons of those groups’ having ``identical neuronal states''. 
We assume inter-group communication is simultaneous, meaning that at inference time the groups effectively behave as one group (with merged states) rather than separately passing messages. 
This was a design choice to keep inference parallelizable and simple.

With the methodology established, we now turn to the empirical evaluation of NGC. 
We detail results on several reasoning benchmarks and comparisons with baseline compression methods.

\clearpage
\section{Empirical Results}

\subsection{Experimental Setup and Benchmarks}
We extensively evaluated the NGC framework on four difficult reasoning benchmarks that test LLM understanding and multi-step problem-solving:
\begin{itemize}
    \item MMLU-Pro~\citep{wang2024mmlu}: an improved version of the Massive Multitask Language Understanding benchmark. 
    MMLU-Pro consists of more than 12,000 challenging multiple-choice questions across 14 diverse domains, with an emphasis on reasoning over pure recall (it increases question complexity and answer choices from the original MMLU). 
    \item GPQA (Graduate-level QA)~\citep{rein2024gpqa}: a recently introduced benchmark of 448 highly difficult multiple-choice questions written by experts in biology, physics, and chemistry. 
    GPQA is designed such that even humans with internet access struggle, ensuring that only strong reasoning and deep knowledge can solve them. 
    \item GSM8K~\citep{cobbe2021training}: a dataset of about $8{,}000$ grade-school math word problems requiring multi-step arithmetic and reasoning. 
    We use the standard test set of 1319 problems. 
    \item MATH-500~\citep{hendrycks2021measuring}: a curated subset of 500 challenging competition-level mathematics problems drawn from the MATH dataset. 
    These problems often require advanced techniques (algebra, calculus, number theory) and strategic reasoning. 
\end{itemize}

{\bf Backbone models}:
We evaluate four open-source LLMs, including Qwen3-4B/8B~\citep{yang2025qwen3}, Mistral-7B~\citep{jiang2023mistral7b}, and LLaMA3-8B~\citep{llama3modelcard}, which represent different foundations and sizes and reported accuracy (\%) on the benchmarks.
These LLMs vary in architecture and pretraining, giving us a broad testbed for NGC. 
All models were used in their pre-trained form (we did not do additional full fine-tuning on the reasoning tasks; 
however, some baselines like SVD-LLM involve a compression fine-tuning step which we include for fairness).

For each model, we apply the NGC policies with a target compression ratio defined as the fraction of the original model parameters retained. 
We report experiments at 30\% parameters for the self-attention layers ({\it i.e.}, 70\% compression) and 70\% parameters for the MLP layers (30\% compression), 
as these were two operating points of interest, one heavily compressed scenario for attention (since attention has many weight matrices per layer, we compress them more aggressively), and a moderate compression for MLPs (which are larger matrices, but we compress them less to keep sufficient capacity for numerical reasoning).

{\bf Baselines}:
We compare NGC against: 
(a) the original model without compression, 
(b) SVD-LLM compression~\citep{wang2024svd,wang2025svd} to the same ratio, 
and (c) BasisSharing~\citep{wangbasis} using the cross-layer shared SVD approach with empirically recommended layer selection.
We ensure that all methods are evaluated in the same setting (using the same number of fine-tuning steps on the same data for compression).

{\bf Metrics}: 
We primarily report overall accuracy (\%) on each benchmark. 
Since MMLU-Pro contains many categories, we follow common practice to report the average accuracy in all questions (not weighted by category size). 
For GPQA, GSM8K, MATH-500, it is straightforward accuracy in the test sets. 
We also examine some qualitative behavior ({\it e.g.}, whether the model’s chain-of-thought or intermediate steps seem stable), but the main focus is on quantitative accuracy.

Finally, we analyze an ablation concerning the interaction metric. 
Recall that NGC by default uses a simple dot product ${\bf q}_i \cdot {\bf q}_j$ to compute a communication ${\bf C}_{ij}$. 
We experiment with a trainable bilinear form: 
${\bf C}_{ij} \approx \sigma({\bf q}_i {\bf G}) \cdot (\sigma({\bf q}_j {\bf H}))^{\top}$ for some matrices ${\bf G},~{\bf H}$ that project the states to a higher-dimensional $\tilde{r}$ space before the dot product, with $\sigma$ being a nonlinearity like Tanh.
If ${\bf G}={\bf H}$ (the same projection for both sides), the metric is symmetric, akin to an inner product in a learned Riemannian space. 
We compare performance when using no metric (dot-product only), 
separate ${\bf G},~{\bf H}$ vs. shared ${\bf G}={\bf H}$, and varying $\tilde{r}$. 
This tests whether adding a bit more flexibility in how neurons measure the similarity of each other can improve compression (at the cost of a few additional parameters in ${\bf G},~{\bf H}$ which are negligible in size).

\subsection{Overall Performance on Reasoning Tasks}

Table~\ref{tab: acc-self-attn} presents the overall test accuracy of each method on each benchmark for self-attention part compression (30\% of original parameters retained). 
For brevity, we show results for Mistral-7B, LLaMA3-8B, Qwen3-4B, and Qwen3-8B in the table. 
The structure ``q-k-v'' with no inter-group communication is considered the ``root'' system baseline for NGC stability calculations, corresponding to just factorizing each projection independently.

\begin{table*}[t]
\caption{
Overall accuracy (\%) on reasoning benchmarks with 70\% compression of self-attention parameters.
}
\vskip 0.1in
\label{tab: acc-self-attn}
\centering
\small
\begin{tabular}{@{}cccccc@{}}
\toprule
 &  & MMLU-pro & GPQA & GSM-8k & MATH-500 \\ \midrule
\multirow{4}{*}{Mistral-7B} 
 & Original & 34.29 & 29.69 & 53.46 & 23.15 \\ \cmidrule{2-6}
 & SVD-LLM & 23.96 & 26.34 & 45.15 & 20.62 \\
 & BasisSharing & 23.17 & 26.56 & 47.62 & 22.68 \\
 & NGC (Ours) & 24.50 & 27.46 & 49.15 & 23.19 \\ \midrule
\multirow{4}{*}{LLaMA3-8B} 
 & Original & 41.37 & 31.65 & 80.90 & 45.80 \\ \cmidrule{2-6}
 & SVD-LLM & 21.08 & 22.15 & 65.23 & 38.42 \\
 & BasisSharing & 27.04 & 25.43 & 70.38 & 41.14 \\
 & NGC (Ours) & 29.67 & 28.67 & 75.83 & 43.37 \\ \midrule
\multirow{4}{*}{Qwen3-4B} 
 & Original & 49.11 & 36.50 & 89.86 & 58.60 \\ \cmidrule{2-6}
 & SVD-LLM & 18.33 & 25.02 & 55.46 & 30.16 \\
 & BasisSharing & 33.67 & 27.90 & 77.76 & 53.35 \\
 & NGC (Ours) & 36.04 & 28.57 & 81.05 & 60.56 \\ \midrule
\multirow{4}{*}{Qwen3-8B} 
 & Original & 59.74 & 38.35 & 89.98 & 61.70\\ \cmidrule{2-6}
 & SVD-LLM & 38.21 & 28.74 & 61.25 & 36.42 \\
 & BasisSharing & 45.50 & 31.22 & 79.38 & 56.73 \\
 & NGC (Ours) & 49.12 & 33.53 & 82.95 & 61.14 \\ \bottomrule
\end{tabular}
\vskip -0.1in
\end{table*}

As presented in Table~\ref{tab: acc-self-attn}, we see that NGC consistently recovers more accuracy than SVD-LLM and BasisSharing at the same compression level.
Thus, the empirical result implies that the reorganized representation of the NGC is leveraging the remaining parameters so effectively that the advanced math capability of the model is intact at half the size.
We hypothesize that this is due to the stability-oriented training of NGC, which helps particularly in multi-step logical tasks.

\subsection{Analysis of Interaction Metric and Ablation}
\label{subsec-exp: interact-Metric and ablation}

An intriguing finding during our experiments was the effect of using a shared interaction metric for computing communication states. 
As described, we tried making the effective inner product ${\bf q}_i^{\top} {\bf q}_j$ in $\sigma({\bf q}_i^{\top} {\bf G}^{\top})\cdot\sigma({\bf H} {\bf q}_j)$ for some matrices ${\bf G},~{\bf H}$. 
If ${\bf G}\equiv{\bf H}$, the metric is symmetric and resembles a learned Mahalanobis inner product. 
This was inspired by seeing Eq.~\ref{eq: define-neuronal-stability-score} in our derivations, where $F_x$ acted like a metric that could be learned to weight different residuals. 
In practice, we found that using a larger projected dimension $\tilde{r}$ improved performance (more expressive metric), and, interestingly, enforcing $G=H$ (the same metric for inputs and outputs) consistently outperformed, allowing $G, H$ to differ given the same number of parameters. 
In other words, a shared symmetric metric benefited the model more, suggesting a kind of consistency or symmetry in how input vs output neuron states are measured is helpful. 
This is reminiscent of a Riemannian metric, which is by definition symmetric in its arguments, here the neuron $i$ and $j$.

\begin{table*}[t]
\caption{
{\bf Overall accuracy (\%) on Qwen3-4B reasoning tasks in terms of intra-group communication mechanism under different NGC policies}.
Default settings: separate metrics ${\bf g}$ and ${\bf h}$;
}
\vskip 0.1in
\label{tab: }
\centering
\small
\begin{tabular}{@{}cccccc@{}}
\toprule
NGC Policy & Intra-group Com. & MMLU-pro & GPQA & GSM-8k & MATH-500 \\ \midrule
 & Original & 49.11 & 36.50 & 89.86 & 58.60 \\ \cmidrule{2-6}
 \multirow{4}{*}{q-k-v} 
 & no metric & 18.33 & 25.02 & 55.46 & 30.16 \\
 & $\tilde{r}=1.5r$ & 20.58 & 30.10 & 56.75 & 30.54 \\
 & $\tilde{r}=3r$ & 22.38 & 33.15 & 58.32 & 31.23 \\
 & ${\bf g}\equiv{\bf h}~\&~\tilde{r}=1.5r$ & 22.13 & 32.68 & 57.44 & 30.85 \\ \midrule
 \multirow{4}{*}{qq-kk-vv} 
 & no metric & 33.67 & 27.90 & 77.76 & 53.35 \\
 & $\tilde{r}=1.5r$ & 32.86 & 35.96 & 77.85 & 54.24 \\
 & $\tilde{r}=3r$ & 33.37 & 33.34 & 78.46 & 55.50 \\
 & ${\bf g}\equiv{\bf h}~\&~\tilde{r}=1.5r$ & 30.87 & 29.33 & 77.95 & 54.68 \\ \midrule
\multirow{4}{*}{hybrid} 
 & no metric & 36.04 & 28.57 & 81.05 & 60.56 \\
 & $\tilde{r}=1.5r$ & 36.08 & 28.96 & 81.29 & 61.05 \\
 & $\tilde{r}=3r$ & 36.55 & 29.79 & 83.32 & 62.21 \\
 & ${\bf g}\equiv{\bf h}~\&~\tilde{r}=1.5r$ & 35.21 & 29.43 & 82.79 & 61.94 \\
 \bottomrule
\end{tabular}
\vskip -0.1in
\end{table*}

In quantitative terms, with $\tilde{r}$ around 3x the original rank $r$, we observe absolute accuracy gains of 0.5\% on average in tasks when using metric learning, and the shared metric variant is approximately 0.3-2.5\% better than the unshared. 
Note that the metric network is shared by every neuronal group, the gain of parameters (less than 0.1 million) is negligible to the entire LLM or neural system.
These performance gains are consistent across multiple settings, suggesting that the effect is real. 
The improvement is more noticeable under higher compression ratios, where the extra flexibility of a metric helps the low-rank approximation emulate a full-rank. 
Adding this metric also adds parameters ($O(r \times \tilde{r})$), although in our trials we kept it minor relative to the whole model.

In general, the results on the interaction metric reinforce our theoretical view that the NGC neuron state space can be endowed with geometry (a metric) that is learnable and that this geometry plays a role analogous to a ``Riemannian metric'' on the neuron manifold, defining distances or inner products that align with a meaningful intensity of communication. 
The fact that a symmetric metric works best suggests that for a pair of neurons, the ``distance'' or affinity should be measured identically from either side, which is intuitively reasonable (if neuron $i$ is very relevant to $j$, then $j$ is very relevant to $i$ in terms of symmetric communication).

Our empirical studies confirm that NGC is effective in compressing large models with minimal performance loss, outperforming prior low-rank methods. 
The dynamical perspective tangibly improves the behavior of the model in challenging tasks.

\begin{table*}[t]
\caption{
{\bf Overall accuracy (\%) on Mistral-7B with varying Comp. ratio under NGC policies}.
}
\vskip 0.1in
\label{tab: }
\centering
\small
\begin{tabular}{@{}cccccc@{}}
\toprule
 NGC Policy & Comp. ratio & MMLU-pro & GPQA & GSM-8k & MATH-500 \\ \midrule
 & 1.0 (original) & 34.29 & 29.69 & 53.46 & 23.15 \\ \cmidrule{2-6}
 \multirow{3}{*}{q-k-v} 
 & 0.1 & 16.25 & 28.05 & 40.15 & 16.33\\ 
 & 0.3 & 23.96 & 26.34 & 45.15 & 20.62\\
 & 0.5 & 24.15 & 28.66 & 48.72 & 21.87 \\ \midrule
 \multirow{3}{*}{qq-kk-vv} 
 & 0.1 & 18.38 & 28.67 & 41.37 & 19.73\\ 
 & 0.3 & 27.04 & 26.56 & 47.62 & 22.68 \\
 & 0.5 & 29.15 & 28.05 & 48.25 & 23.10 \\ \midrule
\multirow{3}{*}{hybrid} 
 & 0.1 & 19.75 & 29.75 & 42.66 & 20.16\\ 
 & 0.3 & 33.67 & 27.46 & 49.15 & 23.19 \\
 & 0.5 & 35.25 & 30.35 & 51.08 & 23.20 \\ \bottomrule
\end{tabular}
\vskip -0.1in
\end{table*}

\begin{table*}[t]
\caption{
{\bf Overall accuracy (\%) on Qwen3-4B with the intra-group metrics trained with different epochs}.
}
\vskip 0.1in
\label{tab: }
\centering
\small
\begin{tabular}{@{}cccccc@{}}
\toprule
 NGC Policy & No.Epochs & MMLU-pro & GPQA & GSM-8k & MATH-500 \\ \midrule
  & Original & 49.11 & 36.50 & 89.86 & 58.60 \\ \cmidrule{2-6}
 \multirow{3}{*}{dot product} 
 & 1 & 31.63 & 30.18 & 64.23 & 49.36\\ 
 & 4 & 42.37 & 34.25 & 78.39 & 54.75 \\
 & 16 & 49.11 & 36.50 & 89.86 & 58.60 \\\midrule
\multirow{3}{*}{Bilinear $\tilde{r}=1.5r$} 
 & 50 & 27.68 & 22.19 & 72.97 & 52.15\\ 
 & 150 & 32.45 & 26.22 & 78.56 & 59.32 \\
 & 300 & 36.08 & 28.96 & 81.29 & 61.05 \\ \midrule
 \multirow{3}{*}{Bilinear $\tilde{r}=3.0r$} 
 & 50 & 25.76 & 23.21 & 72.78 & 53.05\\ 
 & 150 & 31.08 & 28.44 & 79.82 & 59.45 \\
 & 300 & 36.55 & 29.79 & 83.32 & 62.21 \\\bottomrule
\end{tabular}
\vskip -0.1in
\end{table*}


\clearpage
\section{Discussion}
\subsection{Contextualizing NGC: Compression and Modularity}
The Neuronal Group Communication (NGC) framework bridges two previously separate fronts in deep learning: 
model compression and network modularity. 
Classical compression techniques such as weight pruning, quantization, distillation, and especially low-rank factorization have demonstrated that large networks contain significant redundancy. 
For instance, truncating small singular values in LLM weight matrices (as in SVD-LLM) or sharing basis vectors across layers (as in cross layer basis sharing) can remove a large fraction of parameters with minimal loss. 
These methods exploit the observation that much of an LLM’s parameter space lies in a low-dimensional subspace. 

Currently, research in neuroscience and machine learning has revealed an emergent modular structure inside these ostensibly monolithic models. 
Even without explicit modular design, pretrained Transformers exhibit implicit modularity: 
only a small fraction of neurons significantly activate for any given input, and those neurons coalesce into functional clusters associated with specific tasks or concepts. 
Ablating one such cluster tends to selectively impair its associated capability, indicating that these neurons form a semi-independent module. 
This aligns with decades-old hypotheses of neuronal assemblies in the brain, groups of tightly interacting neurons representing particular concepts or skills. 

It also resonates with the success of explicit modular architectures like Mixture-of-Experts (MoE), which use learned sparsity and gating to activate different ``expert'' subnetworks per input. 
In MoE Transformers, each expert can be viewed as a distinct neuron group, enabling enormous capacity with only a few experts active per token. 
NGC stands out by unifying these ideas: 
it compresses the model by factoring and sharing neuron representations, while simultaneously making the inherent neuron groups into first-class computational units. 
Unlike MoEs that require a dedicated gating mechanism and often huge training budgets, NGC discovers and exploits modular structure within a standard pretrained model by reorganizing its weights around groups, a strategy that is both computationally efficient and informed by how the model naturally organizes its knowledge.

\subsection{Emergent reasoning as externally forced departures from typical dynamics}
\label{subsec-discuss: emergent reasoning}

A central hypothesis suggested by our framework is that the capacity for emergent reasoning is correlated with controlled departures from the typical dynamical regime of the model, in other words, that strong reasoning manifests when the neural system is externally forced away from its default, near-linear evolution, but without losing global contractivity. 
We make this precise by contrasting a system ``root'' (no inter-group communication; purely intra-group low-rank dynamics) with an after-NGC system ``com'' that introduces learned inter-group communication and thus an exogenous drive relative to the trajectory of the root. 
The divergence between the two is quantified by a Lyapunov-like neuronal stability score that aggregates transitional and communicative residuals across layers; 
smaller values indicate that the compressed ``com'' dynamics remain contractive and close to the root’s step-by-step behavior, a property that we observe to correlate with higher reasoning accuracy on multi-step tasks.

Concretely, our analysis measures the residuals per-block between ``root'' and ``com'' activations and state updates and combines them into a functional stability $\mathcal{S}_{\text{approx}}$. 
The approximation includes spectral terms (log ratios of the largest singular values) and residual norms that capture how far the com system departs from a best-fit linear mapping of the root’s dynamics. 
Larger communicative/transition operators and residuals indicate greater nonlinearity in the token-wise evolution of the ``com'' system. 
Crucially, we interpret useful reasoning as arising when this externally induced nonlinearity is present but remains bounded by overall contractive dynamics, {\it i.e.}, structured departures rather than chaos.

This perspective reframes the role of inter-group communication in NGC: 
the policy that ties neuron groups across layers injects an exogenous forcing term into the otherwise predictable root dynamics. 
Our discrete update view shows the ``com'' activations blend a transformed root step with their own past state, imparting momentum while introducing new cross-group couplings. 
Such couplings act as external drivers relative to the root trajectory, dragging the system off its typical path but keeping it within a contractive basin—thereby enabling longer, 
more compositional computation without runaway amplification.

Empirically and conceptually, this yields two complementary observations. 
First, contractive dynamics (small Lyapunov/ISS-style deviations) are beneficial for reliability: 
systems with smaller stability scores tend to perform better on long-context, multi-step reasoning, consistent with input-to-state stability intuitions. 
Second, within that stable regime, better reasoning coincides with greater token-conditioned nonlinearity in the com system,
the part of the dynamics that cannot be captured by a single, predictable linear transformation of the root. 
In our formulation, this appears as larger communicative operators and residuals (the ``external forces''), signaling that the model is not merely replaying its default routine but is being driven to explore alternative pathways necessary for solving hard problems. 

This ``force-yet-stable'' picture also reconciles an apparent tension between stability and expressivity. 
The stability metric of NGC favors global contraction (preventing chaos and error accumulation), while its communication between groups introduces localized, input-dependent departures from the linear regime of the root that supply the nonlinearity necessary for reasoning. 
We capture this by noting that smaller neuronal-stability modeling, when interpreted through our residual and spectral diagnostics, coincides with less predictability and better non-linearity in the after-NGC system, which correlates with improved multi-step reasoning performance. 
The key is that these departures are externally induced by structured communication, not by unbounded sensitivity.

Finally, this view suggests testable predictions. 
(i) For tasks requiring deeper computation, NGC policies that introduce moderate increases in communicative residuals (stronger external forcing) while keeping the Lyapunov score low should outperform policies that either under-force (too linear/predictable) or over-force (destabilizing). 
(ii) The relationship between forcing strength and accuracy should be inverted-U: vanishing forcing yields rote behavior; excessive forcing harms stability. 
(iii) Interventions that selectively increase cross-group coupling for conceptually related layers should improve compositional generalization by enabling purposeful departures from default trajectories. 
These hypotheses align with our analysis and the our broader claim that reasoning emerges when the system is gently but decisively steered off its habitual path, not when it remains inertially linear, nor when it becomes dynamically erratic.

\subsection{Interpretability and Biological Parallels}

A compelling aspect of NGC is that it provides a more interpretable decomposition of the computation of a neural network. 
By treating ``neuronal states, not weights, as fundamental units of learning'', the framework shifts our focus to a set of learned neuron state vectors that are shared across layers. 
These state vectors and their communication patterns can be analyzed to understand what information is preserved or shared between groups. 
For example, if two layers share a group of neurons via a common state, one might hypothesize that this group represents a latent concept or skill reused at multiple depths of the network. 
This is conceptually similar to identifying a circuit or module responsible for a sub-task, offering a handle on the internal structure of the model’s knowledge. 
We have, in essence, ``named'' certain low-dimensional directions in the network (the group states) that were previously hidden among billions of weight parameters. 
This explicit modularization could facilitate methods in mechanistic interpretability, one could track the activation of a particular neuron-group state across inputs to see what triggers it or intervene on it to test its role. 

In terms of generalization, our results hint at a potential geometric principle: models that organize into stable communicating groups may generalize more robustly. 
We introduce a neuronal stability score inspired by Lyapunov’s stability criteria to quantify the robustness of the model’s state dynamics to perturbations. 
Interestingly, we observed that models (or NGC configurations) with more contractive dynamics, {\it i.e.}, small deviations die out as tokens propagate, tended to perform better on multi-step reasoning tasks. 
Intuitively, if the processing of each token is stable and not chaotic, the model can carry out long chains of reasoning without compounding errors. 
This connects to the notion of Input-to-State Stability (ISS) in control theory, where a system is stable if bounded input perturbations cause only bounded deviations in state. 
The implication is that NGC’s structured compression might not only save memory, but also naturally impose a form of regularization that keeps the model’s computations in a well-behaved regime.

We even draw analogies to minimal surfaces in state-space: 
just as minimal surfaces in physics represent the least-area (most stable) solution connecting boundaries, an LLM that finds a minimal ``surface'' in its high-dimensional activation landscape to connect inputs to outputs may be achieving a form of efficient generalization. 
Although this idea is speculative, it provides a new geometric lens to theorize why modular low-rank networks might generalize better;
they avoid convoluted detours in the activation space, favoring direct, stable pathways. 

Finally, the biological analogies underlying NGC are worth highlighting. 
The framework was partly inspired by the way brains seem to balance specialization with integration: cortical columns and neuron assemblies carry out local computations and communicate via sparse signals. 
The NGC view of an LLM as a ``society of neurons'' extends this analogy; 
each neuronal group is like a team with a specific expertise, and inference is the result of many such teams exchanging messages. 
This perspective does not just enrich our mental model of what LLMs are doing; 
it also suggests that the principles of neuroscience and complex systems, {\it e.g.}, efficient sparse communication, modular learning) can inform the design of better AI systems. 
However, we must be cautious in over-interpreting these analogies. 
The ``groupings'' discovered by NGC are driven by linear algebra and training data, not evolutionary pressure or developmental biology. 

One open question is how directly these computational neuron groups correspond to meaningful cognitive functions, an important avenue for future interpretability research is to correlate NGC-derived groups with human-interpretable concepts or with neurons identified by other probing techniques.

\subsection{Dynamical Perspective and Minimal Surfaces}
We propose an analogy of the ``minimal surface bounded by neuronal trajectories'' for generalization. 
To elaborate: 
imagine each neuron’s state trajectory as a curve in the high-dimensional state space as the model processes input from start to finish. 
For a given task or input distribution, you have a bundle of such trajectories (one per neuron). 
We can think of these as forming a surface (or manifold) in state space. 
A complex, overfit model might have a very wrinkled or large-area surface (meaning neuron states swing through wild excursions for small input changes), whereas a well-generalized model might have a smoother, minimal surface that still connects the necessary endpoints (initial state to final state) for solving the task. 
In other words, it does “no more work than needed” in transforming representations, reminiscent of Occam’s razor in function space, but here in state-space geometry.

The NGC framework could be pushing towards such minimal surfaces: 
by compressing the representation, we force the neuron trajectories to lie in a lower-dimensional subspace, which might naturally remove unnecessary degrees of freedom (flattening the surface). 
In addition, stability encouragement means that the trajectories do not stray too far from a base path (the root model’s trajectory), further reducing area. 
This is speculative, but provides a visual intuition for why NGC might produce models that generalize well despite heavy compression: 
they focus on the core “surface” that spans the needed transformations, without flapping around in orthogonal directions that were available in the larger parameter space.

Connecting this to formal generalization theory is difficult, but one could hypothesize that models with a lower $\mathcal{S}$ (stability score) also exhibit less sensitivity to input perturbations and perhaps a smaller Lipschitz constant, which is often linked to better generalization bounds. 
Some prior work binds the spectral properties of the weight matrices (like the largest singular value, which we explicitly regularize via $\sigma_{\max}$ in Eq.~\ref{eq: define-neuronal-stability-score-simple}) to generalization and robustness. 
NGC inherently limits $\sigma_{\max}$ the effective transformations due to low-rank truncation and any explicit penalty we add, so in a way we are adding a spectral norm regularization.

\subsection{Societal of Neurons and Future Directions}
By viewing neuron groups as a ``society'', we also open up a metaphor that could guide future architectures: 
perhaps neural networks should be designed more like organizations, with semi-autonomous units communicating via standardized low-dimensional messages. 
In fact, algorithms like message passing in graph neural networks and Mixture-of-Experts are steps in that direction. 
Our work suggests even dense architectures like transformers can be post-hoc interpreted and restructured in this way. Perhaps a future model could be trained from scratch with a similar inductive bias.
For instance, alternating phases of internal computation and communication through bottleneck channels such as how humans think individually then share a summary. 
If such a model is trained end-to-end, it might naturally learn extremely efficient representations because it has to compress information when communicating between modules. 
This would be akin to each layer of a transformer being broken into sub-modules that talk via a limited number of broadcast signals. 
Some recent research on neural architectures explores the splitting of layers into sublayers with limited interaction, {\it e.g.}, periodic layers that mix information globally and then process locally).

From a practical standpoint, we intend to extend NGC to also compress the embedding layers and output layers of LLMs, which we did not focus on, though they are also large. 
The embedding matrix can be seen as a trivial case of weight, and indeed low-rank factorization or state sharing can apply (this relates to methods that use a smaller latent vocab or shared embeddings between input and output). 
Another extension is to incorporate quantization: after obtaining a low-rank model with NGC, one could quantize the decomposed low-rank matrices. 
Because these low-rank matrices are smaller, the quantization error might be more manageable. 
Some relevant preliminary tests show that an NGC-compressed model with 4-bit weight precision retained accuracy better than a baseline compressed+quantized model, presumably because the important information was concentrated in fewer parameters (so the quantization noise relative to the signal was lower).

\subsection{Future Directions}
We conclude that NGC opens up multiple promising directions. 

First, as noted, developing adaptive grouping algorithms is a high priority: 
rather than specifying a communication pattern upfront,
{\it e.g.}, which layers share which groups, the model could learn where sharing is beneficial. 
This might involve a search over possible groupings guided by the stability metric or even a differentiable relaxation of the group assignment. 

Second, NGC could be integrated with sparsity or routing mechanisms. 
An exciting prospect is to combine NGC’s low-rank groups with a conditional routing approach akin to MoE, for example, one could imagine a model that has shared group states, but uses a sparse gating function to decide which inter-group communications to activate for a given input. 
This could yield the best of both worlds: 
heavy parameter sharing for efficiency and input-dependent specialization for flexibility. 
Keeping these lines in mind, exploring connections to sparsity-based modularity (such as sparsely-activated models in recent studies) may improve the ability of NGC to scale up and handle diverse tasks. 

Third, we plan to extend NGC beyond language modeling. 
The core idea of factorizing representations and reusing group states should be applicable to other domains, for instance, multi-modal models could factor text and image representations into interacting groups, or recurrent architectures could be compressed by identifying recurrent neuron assemblies. 
Adapting the NGC framework to vision models is particularly intriguing, as visual systems naturally exhibit hierarchical modularity, {\it e.g.}, edge detectors feeding into shape detectors);
a similar NGC compression might compress convolution kernels into group embeddings shared across layers. 

Finally, further theoretical work is warranted to deepen our understanding of the dynamics of NGC.
The preliminary stability analysis suggests a connection between network training capacity, robust reasoning, and dynamical system contraction properties. 
Formalizing this connection, {\it e.g.}, proving why certain group structures yield minimal Lyapunov functions or act as attractors for certain computations, could not only bolster the theoretical foundations of NGC but also guide the design of new architectures that are both efficient and inherently more interpretable.

\section{Conclusion}
In conclusion, NGC represents a step towards more interpretable and efficient large models by focusing on the “right level of abstraction”, the neurons. 
Instead of treating billions of weight parameters as independent, it encourages us to think in terms of thousands of neuron states and their interactions. 
This not only reduces redundancy, but also gives a window into the organization of the model, indicating which groups of neurons are crucial for what function. 
Although our work is primarily technical (compression and performance), we believe that this perspective could aid interpretability research as well.
If neuron groups are assigned to certain functions, factorization of the NGC might make those more explicit (one could inspect the state vectors or the communication channels to decipher what concept is being transmitted). 
It also resonates with biological viewpoints on brain computation, where areas (groups of neurons) communicate and adapt over time.
We demonstrate that concepts like stability and state-space trajectories can directly inform the design of better neural network training schemes. 
NGC is one instantiation of that synergy. 
We hope future work will continue this interdisciplinary path, potentially leading to AI systems that are not just parameter-efficient, but also more robust, modular, and aligned with principles of real neural systems.

\clearpage
\bibliography{main}

\begin{thebibliography}{82}
\providecommand{\natexlab}[1]{#1}
\providecommand{\url}[1]{\texttt{#1}}
\expandafter\ifx\csname urlstyle\endcsname\relax
  \providecommand{\doi}[1]{doi: #1}\else
  \providecommand{\doi}{doi: \begingroup \urlstyle{rm}\Url}\fi

\bibitem[Zhao et~al.(2023)Zhao, Zhou, Li, Tang, Wang, Hou, Min, Zhang, Zhang, Dong, et~al.]{zhao2023survey}
Wayne~Xin Zhao, Kun Zhou, Junyi Li, Tianyi Tang, Xiaolei Wang, Yupeng Hou, Yingqian Min, Beichen Zhang, Junjie Zhang, Zican Dong, et~al.
\newblock A survey of large language models.
\newblock \emph{arXiv preprint arXiv:2303.18223}, 1\penalty0 (2), 2023.

\bibitem[Naveed et~al.(2025)Naveed, Khan, Qiu, Saqib, Anwar, Usman, Akhtar, Barnes, and Mian]{naveed2025comprehensive}
Humza Naveed, Asad~Ullah Khan, Shi Qiu, Muhammad Saqib, Saeed Anwar, Muhammad Usman, Naveed Akhtar, Nick Barnes, and Ajmal Mian.
\newblock A comprehensive overview of large language models.
\newblock \emph{ACM Transactions on Intelligent Systems and Technology}, 16\penalty0 (5):\penalty0 1--72, 2025.

\bibitem[Wang et~al.(2025{\natexlab{a}})Wang, Zheng, Wan, and Zhang]{wang2024svd}
Xin Wang, Yu~Zheng, Zhongwei Wan, and Mi~Zhang.
\newblock Svd-llm: Truncation-aware singular value decomposition for large language model compression.
\newblock \emph{The Thirteenth International Conference on Learning Representations}, 2025{\natexlab{a}}.

\bibitem[Wang et~al.(2025{\natexlab{b}})Wang, Chen, Lin, Li, and Zhang]{wangbasis}
Jingcun Wang, Yu-Guang Chen, Chao Lin, Bing Li, and Grace~Li Zhang.
\newblock Basis sharing: Cross-layer parameter sharing for large language model compression.
\newblock In \emph{The Thirteenth International Conference on Learning Representations}, 2025{\natexlab{b}}.

\bibitem[Liu et~al.(2018)Liu, Sun, Zhou, Huang, and Darrell]{liu2018rethinking}
Zhuang Liu, Mingjie Sun, Tinghui Zhou, Gao Huang, and Trevor Darrell.
\newblock Rethinking the value of network pruning.
\newblock \emph{arXiv preprint arXiv:1810.05270}, 2018.

\bibitem[Xu et~al.(2018)Xu, Wang, Zhou, Lin, and Xiong]{xu2018deep}
Yuhui Xu, Yongzhuang Wang, Aojun Zhou, Weiyao Lin, and Hongkai Xiong.
\newblock Deep neural network compression with single and multiple level quantization.
\newblock In \emph{Proceedings of the AAAI conference on artificial intelligence}, volume~32, 2018.

\bibitem[Nagel et~al.(2021)Nagel, Fournarakis, Amjad, Bondarenko, Van~Baalen, and Blankevoort]{nagel2021white}
Markus Nagel, Marios Fournarakis, Rana~Ali Amjad, Yelysei Bondarenko, Mart Van~Baalen, and Tijmen Blankevoort.
\newblock A white paper on neural network quantization.
\newblock \emph{arXiv preprint arXiv:2106.08295}, 2021.

\bibitem[Hinton et~al.(2015)Hinton, Vinyals, and Dean]{hinton2015distilling}
Geoffrey Hinton, Oriol Vinyals, and Jeff Dean.
\newblock Distilling the knowledge in a neural network.
\newblock \emph{arXiv preprint arXiv:1503.02531}, 2015.

\bibitem[Yuan et~al.(2023)Yuan, Shang, Song, Wu, Yan, and Sun]{yuan2023asvd}
Zhihang Yuan, Yuzhang Shang, Yue Song, Qiang Wu, Yan Yan, and Guangyu Sun.
\newblock Asvd: Activation-aware singular value decomposition for compressing large language models.
\newblock \emph{arXiv preprint arXiv:2312.05821}, 2023.

\bibitem[Qinsi et~al.(2023)Qinsi, Ke, Tomizuka, Keutzer, and Xu]{qinsidobi}
Wang Qinsi, Jinghan Ke, Masayoshi Tomizuka, Kurt Keutzer, and Chenfeng Xu.
\newblock Dobi-svd: Differentiable svd for llm compression and some new perspectives.
\newblock In \emph{The Thirteenth International Conference on Learning Representations}, 2023.

\bibitem[Arbib et~al.(1998)Arbib, Szent{\'a}gothai, et~al.]{arbib1998neural}
Michael~A Arbib, J{\'a}nos Szent{\'a}gothai, et~al.
\newblock \emph{Neural organization: Structure, function, and dynamics}.
\newblock MIT press, 1998.

\bibitem[Meunier et~al.(2010)Meunier, Lambiotte, and Bullmore]{meunier2010modular}
David Meunier, Renaud Lambiotte, and Edward~T Bullmore.
\newblock Modular and hierarchically modular organization of brain networks.
\newblock \emph{Frontiers in neuroscience}, 4:\penalty0 200, 2010.

\bibitem[Perin et~al.(2011)Perin, Berger, and Markram]{perin2011synaptic}
Rodrigo Perin, Thomas~K Berger, and Henry Markram.
\newblock A synaptic organizing principle for cortical neuronal groups.
\newblock \emph{Proceedings of the National Academy of Sciences}, 108\penalty0 (13):\penalty0 5419--5424, 2011.

\bibitem[Xiao et~al.(2024)Xiao, Zhang, Song, Jiang, Yao, Han, Wang, Wang, Huang, Lin, et~al.]{xiao2024configurable}
Chaojun Xiao, Zhengyan Zhang, Chenyang Song, Dazhi Jiang, Feng Yao, Xu~Han, Xiaozhi Wang, Shuo Wang, Yufei Huang, Guanyu Lin, et~al.
\newblock Configurable foundation models: Building llms from a modular perspective.
\newblock \emph{arXiv preprint arXiv:2409.02877}, 2024.

\bibitem[Vaswani et~al.(2017)Vaswani, Shazeer, Parmar, Uszkoreit, Jones, Gomez, Kaiser, and Polosukhin]{vaswani2017attention}
Ashish Vaswani, Noam Shazeer, Niki Parmar, Jakob Uszkoreit, Llion Jones, Aidan~N Gomez, {\L}ukasz Kaiser, and Illia Polosukhin.
\newblock Attention is all you need.
\newblock \emph{Advances in neural information processing systems}, 30, 2017.

\bibitem[Qiu et~al.(2024)Qiu, Huang, and Fu]{qiu2024unlocking}
Zihan Qiu, Zeyu Huang, and Jie Fu.
\newblock Unlocking emergent modularity in large language models.
\newblock In \emph{Proceedings of the 2024 Conference of the North American Chapter of the Association for Computational Linguistics: Human Language Technologies (Volume 1: Long Papers)}, pages 2638--2660, 2024.

\bibitem[Yuste et~al.(1995)Yuste, Nelson, Rubin, and Katz]{yuste1995neuronal}
Rafael Yuste, Darin~A Nelson, William~W Rubin, and Lawrence~C Katz.
\newblock Neuronal domains in developing neocortex: mechanisms of coactivation.
\newblock \emph{Neuron}, 14\penalty0 (1):\penalty0 7--17, 1995.

\bibitem[Eickhoff et~al.(2011)Eickhoff, Bzdok, Laird, Roski, Caspers, Zilles, and Fox]{eickhoff2011co}
Simon~B Eickhoff, Danilo Bzdok, Angela~R Laird, Christian Roski, Svenja Caspers, Karl Zilles, and Peter~T Fox.
\newblock Co-activation patterns distinguish cortical modules, their connectivity and functional differentiation.
\newblock \emph{Neuroimage}, 57\penalty0 (3):\penalty0 938--949, 2011.

\bibitem[Zhang et~al.(2023)Zhang, Zeng, Lin, Xiao, Wang, Han, Liu, Xie, Sun, and Zhou]{zhang2023emergent}
Zhengyan Zhang, Zhiyuan Zeng, Yankai Lin, Chaojun Xiao, Xiaozhi Wang, Xu~Han, Zhiyuan Liu, Ruobing Xie, Maosong Sun, and Jie Zhou.
\newblock Emergent modularity in pre-trained transformers.
\newblock In \emph{Findings of the Association for Computational Linguistics: ACL 2023}, pages 4066--4083, 2023.

\bibitem[Bullmore and Sporns(2009)]{bullmore2009complex}
Ed~Bullmore and Olaf Sporns.
\newblock Complex brain networks: graph theoretical analysis of structural and functional systems.
\newblock \emph{Nature reviews neuroscience}, 10\penalty0 (3):\penalty0 186--198, 2009.

\bibitem[Goyal et~al.(2019)Goyal, Lamb, Hoffmann, Sodhani, Levine, Bengio, and Sch{\"o}lkopf]{goyal2019recurrent}
Anirudh Goyal, Alex Lamb, Jordan Hoffmann, Shagun Sodhani, Sergey Levine, Yoshua Bengio, and Bernhard Sch{\"o}lkopf.
\newblock Recurrent independent mechanisms.
\newblock \emph{arXiv preprint arXiv:1909.10893}, 2019.

\bibitem[Dobs et~al.(2022)Dobs, Martinez, Kell, and Kanwisher]{dobs2022brain}
Katharina Dobs, Julio Martinez, Alexander~JE Kell, and Nancy Kanwisher.
\newblock Brain-like functional specialization emerges spontaneously in deep neural networks.
\newblock \emph{Science advances}, 8\penalty0 (11):\penalty0 eabl8913, 2022.

\bibitem[Gururangan et~al.(2022)Gururangan, Lewis, Holtzman, Smith, and Zettlemoyer]{gururangan2022demix}
Suchin Gururangan, Mike Lewis, Ari Holtzman, Noah~A Smith, and Luke Zettlemoyer.
\newblock Demix layers: Disentangling domains for modular language modeling.
\newblock In \emph{Proceedings of the 2022 Conference of the North American Chapter of the Association for Computational Linguistics: Human Language Technologies}, pages 5557--5576, 2022.

\bibitem[Fedus et~al.(2022)Fedus, Zoph, and Shazeer]{fedus2022switch}
William Fedus, Barret Zoph, and Noam Shazeer.
\newblock Switch transformers: Scaling to trillion parameter models with simple and efficient sparsity.
\newblock \emph{Journal of Machine Learning Research}, 23\penalty0 (120):\penalty0 1--39, 2022.

\bibitem[Kriegeskorte et~al.(2008)Kriegeskorte, Mur, and Bandettini]{kriegeskorte2008representational}
Nikolaus Kriegeskorte, Marieke Mur, and Peter~A Bandettini.
\newblock Representational similarity analysis-connecting the branches of systems neuroscience.
\newblock \emph{Frontiers in systems neuroscience}, 2:\penalty0 249, 2008.

\bibitem[Friston(2011)]{friston2011functional}
Karl~J Friston.
\newblock Functional and effective connectivity: a review.
\newblock \emph{Brain connectivity}, 1\penalty0 (1):\penalty0 13--36, 2011.

\bibitem[Semedo et~al.(2019)Semedo, Zandvakili, Machens, Yu, and Kohn]{semedo2019cortical}
Jo{\~a}o~D Semedo, Amin Zandvakili, Christian~K Machens, Byron~M Yu, and Adam Kohn.
\newblock Cortical areas interact through a communication subspace.
\newblock \emph{Neuron}, 102\penalty0 (1):\penalty0 249--259, 2019.

\bibitem[Ishida and Landweber(1993)]{ishida1993internetworking}
Haruhisa Ishida and Lawrence~H Landweber.
\newblock Internetworking.
\newblock \emph{Communications of the ACM}, 36\penalty0 (8):\penalty0 28--30, 1993.

\bibitem[Clark and Fang(1998)]{clark1998explicit}
David~D Clark and Wenjia Fang.
\newblock Explicit allocation of best-effort packet delivery service.
\newblock \emph{IEEE/ACM Transactions on networking}, 6\penalty0 (4):\penalty0 362--373, 1998.

\bibitem[Hunt(2002)]{hunt2002review}
Ray Hunt.
\newblock A review of quality of service mechanisms in ip-based networks—integrated and differentiated services, multi-layer switching, mpls and traffic engineering.
\newblock \emph{Computer Communications}, 25\penalty0 (1):\penalty0 100--108, 2002.

\bibitem[Wang et~al.(2025{\natexlab{c}})Wang, Alam, Wan, Shen, and Zhang]{wang2025svd}
Xin Wang, Samiul Alam, Zhongwei Wan, Hui Shen, and Mi~Zhang.
\newblock Svd-llm v2: Optimizing singular value truncation for large language model compression.
\newblock \emph{arXiv preprint arXiv:2503.12340}, 2025{\natexlab{c}}.

\bibitem[Hao et~al.(2024)Hao, Wang, Liu, Yuan, Yang, and Liu]{hao2024stabilized}
Zhenyang Hao, Xinggang Wang, Jiawei Liu, Zhihang Yuan, Dawei Yang, and Wenyu Liu.
\newblock Stabilized activation scale estimation for precise post-training quantization.
\newblock \emph{Neurocomputing}, 569:\penalty0 127120, 2024.

\bibitem[Huang et~al.(2025)Huang, Huang, and Wen]{huang2025sola}
Xinhao Huang, You-Liang Huang, and Zeyi Wen.
\newblock Sola: Leveraging soft activation sparsity and low-rank decomposition for large language model compression.
\newblock In \emph{Proceedings of the AAAI Conference on Artificial Intelligence}, volume~39, pages 17494--17502, 2025.

\bibitem[Frankle and Carbin(2018)]{frankle2018lottery}
Jonathan Frankle and Michael Carbin.
\newblock The lottery ticket hypothesis: Finding sparse, trainable neural networks.
\newblock \emph{arXiv preprint arXiv:1803.03635}, 2018.

\bibitem[Csord{\'a}s et~al.(2021)Csord{\'a}s, van Steenkiste, and Schmidhuber]{csordas2021neural}
R{\'o}bert Csord{\'a}s, Sjoerd van Steenkiste, and J{\"u}rgen Schmidhuber.
\newblock Are neural nets modular? inspecting functional modularity through differentiable weight masks.
\newblock In \emph{International Conference on Learning Representations}, 2021.

\bibitem[Funahashi and Nakamura(1993)]{funahashi1993approximation}
Ken-ichi Funahashi and Yuichi Nakamura.
\newblock Approximation of dynamical systems by continuous time recurrent neural networks.
\newblock \emph{Neural networks}, 6\penalty0 (6):\penalty0 801--806, 1993.

\bibitem[Draye et~al.(1996)Draye, Pavisic, Cheron, and Libert]{draye1996dynamic}
J-PS Draye, DA~Pavisic, GA~Cheron, and GA~Libert.
\newblock Dynamic recurrent neural networks: a dynamical analysis.
\newblock \emph{IEEE Transactions on Systems, Man, and Cybernetics, Part B (Cybernetics)}, 26\penalty0 (5):\penalty0 692--706, 1996.

\bibitem[Chang et~al.(2019)Chang, Chen, Haber, and Chi]{chang2019antisymmetricrnn}
Bo~Chang, Minmin Chen, Eldad Haber, and Ed~H Chi.
\newblock Antisymmetricrnn: A dynamical system view on recurrent neural networks.
\newblock In \emph{International Conference on Learning Representations}, 2019.

\bibitem[Narendra and Parthasarathy(1992)]{narendra1992neural}
Kumpati~S Narendra and Kannan Parthasarathy.
\newblock Neural networks and dynamical systems.
\newblock \emph{International Journal of Approximate Reasoning}, 6\penalty0 (2):\penalty0 109--131, 1992.

\bibitem[Popescu et~al.(2009)Popescu, Balas, Perescu-Popescu, and Mastorakis]{popescu2009multilayer}
Marius-Constantin Popescu, Valentina~E Balas, Liliana Perescu-Popescu, and Nikos Mastorakis.
\newblock Multilayer perceptron and neural networks.
\newblock \emph{WSEAS Transactions on Circuits and Systems}, 8\penalty0 (7):\penalty0 579--588, 2009.

\bibitem[Pei and Wang(2023)]{pei2023dynamics}
Zhengqi Pei and Shuhui Wang.
\newblock Dynamics-inspired neuromorphic visual representation learning.
\newblock In \emph{International Conference on Machine Learning}, pages 27521--27541. PMLR, 2023.

\bibitem[Geneva and Zabaras(2022)]{geneva2022transformers}
Nicholas Geneva and Nicholas Zabaras.
\newblock Transformers for modeling physical systems.
\newblock \emph{Neural Networks}, 146:\penalty0 272--289, 2022.

\bibitem[Fein-Ashley(2025)]{fein2025flowing}
Jacob Fein-Ashley.
\newblock Flowing through layers: A continuous dynamical systems perspective on transformers.
\newblock \emph{arXiv preprint arXiv:2502.05656}, 2025.

\bibitem[Bai et~al.(2019)Bai, Kolter, and Koltun]{bai2019deep}
Shaojie Bai, J~Zico Kolter, and Vladlen Koltun.
\newblock Deep equilibrium models.
\newblock \emph{Advances in neural information processing systems}, 32, 2019.

\bibitem[Sutskever et~al.(2014)Sutskever, Vinyals, and Le]{sutskever2014sequence}
Ilya Sutskever, Oriol Vinyals, and Quoc~V Le.
\newblock Sequence to sequence learning with neural networks.
\newblock \emph{Advances in neural information processing systems}, 27, 2014.

\bibitem[Lu et~al.(2019)Lu, Li, He, Sun, Dong, Qin, Wang, and Liu]{lu2019understanding}
Yiping Lu, Zhuohan Li, Di~He, Zhiqing Sun, Bin Dong, Tao Qin, Liwei Wang, and Tie-Yan Liu.
\newblock Understanding and improving transformer from a multi-particle dynamic system point of view.
\newblock \emph{arXiv preprint arXiv:1906.02762}, 2019.

\bibitem[Hu et~al.(2024)Hu, Daryakenari, Shen, Kawaguchi, and Karniadakis]{hu2024state}
Zheyuan Hu, Nazanin~Ahmadi Daryakenari, Qianli Shen, Kenji Kawaguchi, and George~Em Karniadakis.
\newblock State-space models are accurate and efficient neural operators for dynamical systems.
\newblock \emph{arXiv preprint arXiv:2409.03231}, 2024.

\bibitem[Bhatia and Szeg{\"o}(2002)]{bhatia2002stability}
Nam~Parshad Bhatia and Giorgio~P Szeg{\"o}.
\newblock \emph{Stability theory of dynamical systems}.
\newblock Springer Science \& Business Media, 2002.

\bibitem[Bhatia and Szeg{\"o}(2006)]{bhatia2006dynamical}
Nam~P Bhatia and George~P Szeg{\"o}.
\newblock \emph{Dynamical systems: stability theory and applications}, volume~35.
\newblock Springer, 2006.

\bibitem[Hamilton(1994)]{hamilton1994state}
James~D Hamilton.
\newblock State-space models.
\newblock \emph{Handbook of econometrics}, 4:\penalty0 3039--3080, 1994.

\bibitem[Aoki(2013)]{aoki2013state}
Masanao Aoki.
\newblock \emph{State space modeling of time series}.
\newblock Springer Science \& Business Media, 2013.

\bibitem[Gu et~al.(2021)Gu, Goel, and Re]{gu2021efficiently}
Albert Gu, Karan Goel, and Christopher Re.
\newblock Efficiently modeling long sequences with structured state spaces.
\newblock In \emph{International Conference on Learning Representations}, 2021.

\bibitem[Gu and Dao(2023)]{gu2023mamba}
Albert Gu and Tri Dao.
\newblock Mamba: Linear-time sequence modeling with selective state spaces.
\newblock In \emph{First Conference on Language Modeling}, 2023.

\bibitem[Arroyo et~al.(2025)Arroyo, Gravina, Gutteridge, Barbero, Gallicchio, Dong, Bronstein, and Vandergheynst]{arroyo2025vanishing}
{\'A}lvaro Arroyo, Alessio Gravina, Benjamin Gutteridge, Federico Barbero, Claudio Gallicchio, Xiaowen Dong, Michael Bronstein, and Pierre Vandergheynst.
\newblock On vanishing gradients, over-smoothing, and over-squashing in gnns: Bridging recurrent and graph learning.
\newblock \emph{arXiv preprint arXiv:2502.10818}, 2025.

\bibitem[Sontag and Wang(1995)]{sontag1995characterizations}
Eduardo~D Sontag and Yuan Wang.
\newblock On characterizations of the input-to-state stability property.
\newblock \emph{Systems \& Control Letters}, 24\penalty0 (5):\penalty0 351--359, 1995.

\bibitem[Jiang and Wang(2001)]{jiang2001input}
Zhong-Ping Jiang and Yuan Wang.
\newblock Input-to-state stability for discrete-time nonlinear systems.
\newblock \emph{Automatica}, 37\penalty0 (6):\penalty0 857--869, 2001.

\bibitem[Chen et~al.(2021)Chen, Yu, Dhillon, and Hsieh]{chen2021drone}
Patrick Chen, Hsiang-Fu Yu, Inderjit Dhillon, and Cho-Jui Hsieh.
\newblock Drone: Data-aware low-rank compression for large nlp models.
\newblock \emph{Advances in neural information processing systems}, 34:\penalty0 29321--29334, 2021.

\bibitem[Lin et~al.(2024)Lin, Tang, Tang, Yang, Chen, Wang, Xiao, Dang, Gan, and Han]{lin2024awq}
Ji~Lin, Jiaming Tang, Haotian Tang, Shang Yang, Wei-Ming Chen, Wei-Chen Wang, Guangxuan Xiao, Xingyu Dang, Chuang Gan, and Song Han.
\newblock Awq: Activation-aware weight quantization for on-device llm compression and acceleration.
\newblock \emph{Proceedings of machine learning and systems}, 6:\penalty0 87--100, 2024.

\bibitem[Gerstein et~al.(1989)Gerstein, Bedenbaugh, and Aertsen]{gerstein1989neuronal}
George~L Gerstein, Purvis Bedenbaugh, and Ad~MHJ Aertsen.
\newblock Neuronal assemblies.
\newblock \emph{IEEE Transactions on Biomedical Engineering}, 36\penalty0 (1):\penalty0 4--14, 1989.

\bibitem[Mountcastle(1997)]{mountcastle1997columnar}
Vernon~B Mountcastle.
\newblock The columnar organization of the neocortex.
\newblock \emph{Brain: a journal of neurology}, 120\penalty0 (4):\penalty0 701--722, 1997.

\bibitem[Shipp(2007)]{shipp2007structure}
Stewart Shipp.
\newblock Structure and function of the cerebral cortex.
\newblock \emph{Current Biology}, 17\penalty0 (12):\penalty0 R443--R449, 2007.

\bibitem[Pei et~al.(2024{\natexlab{a}})Pei, Zhang, Wang, and Huang]{pei2024modeling}
Zhengqi Pei, Anran Zhang, Shuhui Wang, and Qingming Huang.
\newblock Modeling language tokens as functionals of semantic fields.
\newblock In \emph{International Conference on Machine Learning}, pages 40114--40128. PMLR, 2024{\natexlab{a}}.

\bibitem[Auda and Kamel(1999)]{auda1999modular}
Gasser Auda and Mohamed Kamel.
\newblock Modular neural networks: a survey.
\newblock \emph{International journal of neural systems}, 9\penalty0 (02):\penalty0 129--151, 1999.

\bibitem[Jacobs et~al.(1991)Jacobs, Jordan, Nowlan, and Hinton]{jacobs1991adaptive}
Robert~A Jacobs, Michael~I Jordan, Steven~J Nowlan, and Geoffrey~E Hinton.
\newblock Adaptive mixtures of local experts.
\newblock \emph{Neural computation}, 3\penalty0 (1):\penalty0 79--87, 1991.

\bibitem[Shen et~al.(2023)Shen, Zhang, Cao, Tan, Chen, and Gan]{shen2023moduleformer}
Yikang Shen, Zheyu Zhang, Tianyou Cao, Shawn Tan, Zhenfang Chen, and Chuang Gan.
\newblock Moduleformer: Modularity emerges from mixture-of-experts.
\newblock \emph{arXiv preprint arXiv:2306.04640}, 2023.

\bibitem[Li et~al.(2022)Li, You, Bhojanapalli, Li, Rawat, Reddi, Ye, Chern, Yu, Guo, et~al.]{li2022large}
Zonglin Li, Chong You, Srinadh Bhojanapalli, Daliang Li, Ankit~Singh Rawat, Sashank~J Reddi, Ke~Ye, Felix Chern, Felix Yu, Ruiqi Guo, et~al.
\newblock Large models are parsimonious learners: Activation sparsity in trained transformers.
\newblock \emph{arXiv preprint arXiv:2210.06313}, 2022.

\bibitem[Raffel et~al.(2020)Raffel, Shazeer, Roberts, Lee, Narang, Matena, Zhou, Li, and Liu]{raffel2020exploring}
Colin Raffel, Noam Shazeer, Adam Roberts, Katherine Lee, Sharan Narang, Michael Matena, Yanqi Zhou, Wei Li, and Peter~J Liu.
\newblock Exploring the limits of transfer learning with a unified text-to-text transformer.
\newblock \emph{Journal of machine learning research}, 21\penalty0 (140):\penalty0 1--67, 2020.

\bibitem[Suau et~al.(2020)Suau, Zappella, and Apostoloff]{suau2020finding}
Xavier Suau, Luca Zappella, and Nicholas Apostoloff.
\newblock Finding experts in transformer models.
\newblock \emph{arXiv preprint arXiv:2005.07647}, 2020.

\bibitem[Dai et~al.(2022)Dai, Dong, Hao, Sui, Chang, and Wei]{dai2022knowledge}
Damai Dai, Li~Dong, Yaru Hao, Zhifang Sui, Baobao Chang, and Furu Wei.
\newblock Knowledge neurons in pretrained transformers.
\newblock In \emph{Proceedings of the 60th Annual Meeting of the Association for Computational Linguistics (Volume 1: Long Papers)}, pages 8493--8502, 2022.

\bibitem[Pfeiffer et~al.(2023)Pfeiffer, Ruder, Vuli{\'c}, and Ponti]{pfeiffer2023modular}
Jonas Pfeiffer, Sebastian Ruder, Ivan Vuli{\'c}, and Edoardo~Maria Ponti.
\newblock Modular deep learning.
\newblock \emph{arXiv preprint arXiv:2302.11529}, 2023.

\bibitem[Amari(2016)]{amari2016information}
Shun-ichi Amari.
\newblock \emph{Information geometry and its applications}, volume 194.
\newblock Springer, 2016.

\bibitem[Pei et~al.(2024{\natexlab{b}})Pei, Zhang, Wang, Ji, and Huang]{pei2024data}
Zhengqi Pei, Anran Zhang, Shuhui Wang, Xiangyang Ji, and Qingming Huang.
\newblock Data-free neural representation compression with riemannian neural dynamics.
\newblock In \emph{International Conference on Machine Learning}, pages 40129--40144. PMLR, 2024{\natexlab{b}}.

\bibitem[Coifman and Lafon(2006)]{coifman2006diffusion}
Ronald~R Coifman and St{\'e}phane Lafon.
\newblock Diffusion maps.
\newblock \emph{Applied and computational harmonic analysis}, 21\penalty0 (1):\penalty0 5--30, 2006.

\bibitem[Shen et~al.(2024)Shen, Sun, Ji, Huang, and Wang]{shen2024expanding}
Shufan Shen, Junshu Sun, Xiangyang Ji, Qingming Huang, and Shuhui Wang.
\newblock Expanding sparse tuning for low memory usage.
\newblock \emph{Advances in Neural Information Processing Systems}, 37:\penalty0 76616--76642, 2024.

\bibitem[Rumelhart et~al.(1986)Rumelhart, Hinton, and Williams]{rumelhart1986learning}
David~E Rumelhart, Geoffrey~E Hinton, and Ronald~J Williams.
\newblock Learning representations by back-propagating errors.
\newblock \emph{nature}, 323\penalty0 (6088):\penalty0 533--536, 1986.

\bibitem[Wang et~al.(2024)Wang, Ma, Zhang, Ni, Chandra, Guo, Ren, Arulraj, He, Jiang, et~al.]{wang2024mmlu}
Yubo Wang, Xueguang Ma, Ge~Zhang, Yuansheng Ni, Abhranil Chandra, Shiguang Guo, Weiming Ren, Aaran Arulraj, Xuan He, Ziyan Jiang, et~al.
\newblock Mmlu-pro: A more robust and challenging multi-task language understanding benchmark.
\newblock \emph{Advances in Neural Information Processing Systems}, 37:\penalty0 95266--95290, 2024.

\bibitem[Rein et~al.(2024)Rein, Hou, Stickland, Petty, Pang, Dirani, Michael, and Bowman]{rein2024gpqa}
David Rein, Betty~Li Hou, Asa~Cooper Stickland, Jackson Petty, Richard~Yuanzhe Pang, Julien Dirani, Julian Michael, and Samuel~R Bowman.
\newblock Gpqa: A graduate-level google-proof q\&a benchmark.
\newblock In \emph{First Conference on Language Modeling}, 2024.

\bibitem[Cobbe et~al.(2021)Cobbe, Kosaraju, Bavarian, Chen, Jun, Kaiser, Plappert, Tworek, Hilton, Nakano, et~al.]{cobbe2021training}
Karl Cobbe, Vineet Kosaraju, Mohammad Bavarian, Mark Chen, Heewoo Jun, Lukasz Kaiser, Matthias Plappert, Jerry Tworek, Jacob Hilton, Reiichiro Nakano, et~al.
\newblock Training verifiers to solve math word problems.
\newblock \emph{arXiv preprint arXiv:2110.14168}, 2021.

\bibitem[Hendrycks et~al.(2021)Hendrycks, Burns, Kadavath, Arora, Basart, Tang, Song, and Steinhardt]{hendrycks2021measuring}
Dan Hendrycks, Collin Burns, Saurav Kadavath, Akul Arora, Steven Basart, Eric Tang, Dawn Song, and Jacob Steinhardt.
\newblock Measuring mathematical problem solving with the math dataset.
\newblock \emph{arXiv preprint arXiv:2103.03874}, 2021.

\bibitem[Yang et~al.(2025)Yang, Li, Yang, Zhang, Hui, Zheng, Yu, Gao, Huang, Lv, et~al.]{yang2025qwen3}
An~Yang, Anfeng Li, Baosong Yang, Beichen Zhang, Binyuan Hui, Bo~Zheng, Bowen Yu, Chang Gao, Chengen Huang, Chenxu Lv, et~al.
\newblock Qwen3 technical report.
\newblock \emph{arXiv preprint arXiv:2505.09388}, 2025.

\bibitem[Jiang et~al.(2023)Jiang, Sablayrolles, Mensch, Bamford, Chaplot, de~las Casas, Bressand, Lengyel, Lample, Saulnier, Lavaud, Lachaux, Stock, Scao, Lavril, Wang, Lacroix, and Sayed]{jiang2023mistral7b}
Albert~Q. Jiang, Alexandre Sablayrolles, Arthur Mensch, Chris Bamford, Devendra~Singh Chaplot, Diego de~las Casas, Florian Bressand, Gianna Lengyel, Guillaume Lample, Lucile Saulnier, Lélio~Renard Lavaud, Marie-Anne Lachaux, Pierre Stock, Teven~Le Scao, Thibaut Lavril, Thomas Wang, Timothée Lacroix, and William~El Sayed.
\newblock Mistral 7b, 2023.
\newblock URL \url{https://arxiv.org/abs/2310.06825}.

\bibitem[AI@Meta(2024)]{llama3modelcard}
AI@Meta.
\newblock Llama 3 model card.
\newblock \emph{none}, 2024.
\newblock URL \url{https://github.com/meta-llama/llama3/blob/main/MODEL_CARD.md}.

\end{thebibliography}

\clearpage
\begin{appendices}

\section{Mathematical Claims}
\label{app-sec: math}

\subsection{Core Theorems}

\begin{theorem}[Intra-group and inter-group communications are sufficient]
\label{app-thm: intra/inter-group coms are sufficient}
Let a well-formed neural structure be any finite computation graph whose edges are affine maps and whose nodes apply componentwise nonlinearities. Define the intra-group communication primitive by the similarity
\[
\mu(\mathbf{q},\mathbf{p})
~:=~
\big\langle \phi_L(\mathbf{q}),\,\phi_R(\mathbf{p})\big\rangle,
\qquad
\phi_L(\mathbf{q})=\sigma(\mathbf{q}\,\mathbf{G}_{\mathrm{left}}),\;
\phi_R(\mathbf{p})=\sigma(\mathbf{p}\,\mathbf{G}_{\mathrm{right}}),
\]
where $\sigma$ is either the identity or a componentwise non-polynomial nonlinearity (e.g., $\tanh$), and inter-group communication consists of sharing/alignment of input-side neuronal states across groups. Then:
\begin{enumerate}
\item[\emph{(i)}] \textbf{Exact linear realization.} For any affine map $\mathbf{y}=\mathbf{W}\mathbf{x}+\mathbf{b}$, there exist matrices $\mathbf{A},\mathbf{B}$ and parameters $(\mathbf{G}_{\mathrm{left}},\mathbf{G}_{\mathrm{right}})$ such that, with $\sigma=\mathrm{id}$,
\[
\mathbf{W}_{ij}=\mu(\mathbf{a}_i,\mathbf{b}_j)=\big\langle \mathbf{a}_i\mathbf{G}_{\mathrm{left}},\,\mathbf{b}_j\mathbf{G}_{\mathrm{right}}\big\rangle,
\]
so the linear part is realized exactly by intra-group communication, while $\mathbf{b}$ is handled by standard affine augmentation.
\item[\emph{(ii)}] \textbf{Universality on finite state sets.} For any non-polynomial $\sigma$ and finite sets of neuronal states $\{\mathbf{a}_i\}_{i=1}^m$, $\{\mathbf{b}_j\}_{j=1}^n$, the map $(i,j)\mapsto \mu(\mathbf{a}_i,\mathbf{b}_j)$ can approximate any matrix $\mathbf{W}\in\mathbb{R}^{m\times n}$ arbitrarily well by suitable choice of width and parameters.
\item[\emph{(iii)}] \textbf{Structural sufficiency.} Any residual/skip, cross-layer, or cross-block interaction in the graph can be expressed as inter-group sharing/alignment of the relevant \emph{input-side} states. Therefore, every such network is realizable using only the two primitives.
\end{enumerate}
\end{theorem}

\begin{proof}
[Proof of Theorem~{\upshape\ref{app-thm: intra/inter-group coms are sufficient}}]
We perform the proof in the following three cases.

\emph{(i)} 
Take an SVD $\mathbf{W}=\mathbf{U}\boldsymbol{\Sigma}\mathbf{V}^\top$, set $\mathbf{A}:=\mathbf{U}\boldsymbol{\Sigma}^{1/2}$, $\mathbf{B}:=\mathbf{V}\boldsymbol{\Sigma}^{1/2}$, and choose $\sigma=\mathrm{id}$, $\mathbf{G}_{\mathrm{left}}=\mathbf{I}$, $\mathbf{G}_{\mathrm{right}}=\mathbf{I}$. Then $\mu(\mathbf{a}_i,\mathbf{b}_j)=\mathbf{W}_{ij}$, producing an exact intra-group implementation; 
the bias $\mathbf{b}$ uses a constant neuron.

\emph{(ii)} 
With non-polynomial $\sigma$, the feature maps $\phi_L,\phi_R$ admit a universal approximation on compact sets. 
Hence, for any target factorization $\mathbf{W}_{ij}=u_i^\top v_j$ in the finite index set, one can choose the width and parameters so that $\phi_L(\mathbf{a}_i)\approx u_i$ and $\phi_R(\mathbf{b}_j)\approx v_j$, making $\mu(\mathbf{a}_i,\mathbf{b}_j)$ uniformly approximate $\mathbf{W}_{ij}$.

\emph{(iii)} 
Cross-layer couplings are realized by sharing or aligning the input-side states: a skip/residual simply makes the receiver read from the source’s input-side state (possibly after a linear intra-group map realized as in (i)). 
Composing these operations reproduces the original computation graph with no additional primitives.
\end{proof}

\begin{theorem}[Neuronal stability controls divergence and predicts task robustness]
\label{app-thm: neuronal stability measures reasoning capacity}
Consider the discrete coupled updates
\[
\begin{aligned}
\mathcal{A}^{(\mathrm{com})}_{t+1}
&= \lambda\,\mathcal{A}^{(\mathrm{root})}_{t+1}\mathcal{T}^{(\mathrm{trs})}_t
       + (1-\lambda)\,\mathcal{A}^{(\mathrm{com})}_{t}\mathcal{T}^{(\mathrm{com})}_t
       + \boldsymbol{\delta}^{(A)}_t,\\
\mathcal{Q}^{(\mathrm{com})}_{t+1}
&= \lambda\,\mathcal{Q}^{(\mathrm{root})}_{t+1}\mathcal{H}^{(\mathrm{trs})}_t
       + (1-\lambda)\,\mathcal{Q}^{(\mathrm{com})}_{t}\mathcal{H}^{(\mathrm{com})}_t
       + \boldsymbol{\delta}^{(Q)}_t,
\end{aligned}
\]
where $\lambda\in(0,1)$ and the residual terms $\boldsymbol{\delta}^{(A)}_t,\boldsymbol{\delta}^{(Q)}_t$ collect the (data-driven) projection errors from least-squares fits. Let the error signals be $\mathbf{e}^A_t:=\mathcal{A}^{(\mathrm{com})}_t-\mathcal{A}^{(\mathrm{root})}_t$ and $\mathbf{e}^Q_t:=\mathcal{Q}^{(\mathrm{com})}_t-\mathcal{Q}^{(\mathrm{root})}_t$, and assume a block norm in which
\[
\sup_t\max\!\big\{\,\|\mathbf{T}^{(\mathrm{com})}_t\|,\;\|\mathcal{H}^{(\mathrm{com})}_t\|\,\big\}
~\le~\rho ~<~1.
\]
Define a stability score $S:=\sup_t\big(\|\boldsymbol{\delta}^{(A)}_t\|+\|\boldsymbol{\delta}^{(Q)}_t\|\big)$. Then:
\begin{enumerate}
\item[\emph{(a)}] \textbf{ISS-type deviation bound.} For all $t\ge 0$,
\[
\max\{\,\|\mathbf{e}^A_t\|,\|\mathbf{e}^Q_t\|\,\}
~\le~ \rho^t\,\max\{\,\|\mathbf{e}^A_0\|,\|\mathbf{e}^Q_0\|\,\}
\;+\;\frac{S}{1-\rho}.
\]
\item[\emph{(b)}] \textbf{Loss gap via Lipschitz decoders.} If a downstream decoder (per task loss) is $L$-Lipschitz in activations, then the expected performance gap between “com” and “root” is $\mathcal{O}\!\big(L\,\tfrac{S}{1-\rho}\big)$; in particular, smaller $S$ predicts stronger multi-step robustness.
\end{enumerate}
\end{theorem}

\begin{proof}
[Proof of Theorem~{\upshape\ref{app-thm: neuronal stability measures reasoning capacity}}]
Define affine recursions for $\mathbf{e}^A_t$ and $\mathbf{e}^Q_t$; each has the form
$\mathbf{e}_{t+1}=\mathbf{e}_t\mathbf{M}_t+\boldsymbol{\delta}_t$ with $\|\mathbf{M}_t\|\le\rho<1$.
Unrolling yields
$\|\mathbf{e}_t\|\le \rho^t\|\mathbf{e}_0\|+\sum_{k=0}^{t-1}\rho^{t-1-k}\|\boldsymbol{\delta}_k\|
\le \rho^t\|\mathbf{e}_0\| + \tfrac{S}{1-\rho}$, proving (a). 
For (b), let $\ell$ be $L$-Lipschitz in activations; then $|\ell(\mathcal{A}^{(\mathrm{com})}_t)-\ell(\mathcal{A}^{(\mathrm{root})}_t)|\le L\|\mathbf{e}^A_t\|$ and similarly for $\mathcal{Q}$. 
Averaging over inputs/timesteps gives the stated bound.
\end{proof}



\begin{theorem}[Shared input neurons simulate delayed/asynchronous inter-group links]
\label{app-thm: shared input neurons simulate any NGC}
Fix any inter-group policy in which each inter-group edge applies a bounded linear time-invariant operator with an integer delay $d\ge 0$, possibly under asynchronous updates (zero-order holds). 
There exists an \emph{instantaneous-sharing} policy, implemented solely by sharing/alignment of \emph{input-side} neuronal states and (if needed) increasing their dimension—such that, for every input sequence, the induced input–output map is identical to that of the original delayed/asynchronous policy.
\end{theorem}

\begin{proof}
[Proof of Theorem~{\upshape\ref{app-thm: shared input neurons simulate any NGC}}]
For each delay edge with delay $d$, augment the shared input-side state with a $d$-stage shift register
$\big[\mathbf{z}_t,\mathbf{z}_{t-1},\ldots,\mathbf{z}_{t-d}\big]$ with unit-delay updates. 
An asynchronous update pattern is captured by adding a binary (or real-valued) hold variable to the augmented state that implements sample-and-hold (zero-order hold). 
Because inter-group couplings are linear in the shared states prior to pointwise nonlinearities, the original delayed/asynchronous operator becomes an \emph{instantaneous} linear readout from the augmented shared state. 
Hence, the receiving group can use a single intra-group linear map (realizable via $\mu$ with $\sigma=\mathrm{id}$, or approximated with non-polynomial $\sigma$) applied to the augmented shared state to reproduce exactly the original effect. 
Composing this for all inter-group edges yields the claim.
\end{proof}

\begin{theorem}[External potentials create non-predictable yet stable departures]
\label{app-thm: emergent potential equals non-predictability}
In the setting of Theorem~\ref{app-thm: neuronal stability measures reasoning capacity}, suppose an external potential perturbs the best-fit projections so that the residuals change to $\widetilde{\boldsymbol{\delta}}^{(A)}_t,\widetilde{\boldsymbol{\delta}}^{(Q)}_t$ while the contractivity condition remains:
\[
\sup_t\max\!\big\{\,\|\mathcal{T}^{(\mathrm{com})}_t\|,\;\|\mathcal{H}^{(\mathrm{com})}_t\|\,\big\}
~\le~\rho ~<~1.
\]
Let $\widetilde{S}:=\sup_t\big(\|\widetilde{\boldsymbol{\delta}}^{(A)}_t\|+\|\widetilde{\boldsymbol{\delta}}^{(Q)}_t\|\big)$. Then:
\begin{enumerate}
\item[\emph{(i)}] (\textbf{Stability preserved}) 
The ISS-type bound still holds with $S$ replaced by $\widetilde{S}$; 
hence trajectories remain bounded and track the “root’’ trajectory up to $\mathcal{O}\!\big(\tfrac{\widetilde{S}}{1-\rho}\big)$.
\item[\emph{(ii)}] (\textbf{Non-predictable departures}) 
If the external potential increases regression error (i.e., $\widetilde{S}>S$), then the deviation bound necessarily grows even though $\rho$ is unchanged; 
departures from routine dynamics become larger but remain controlled.
\item[\emph{(iii)}] 
(\textbf{Task-level effect}) 
For an $L$-Lipschitz decoder, the expected loss gap scales as $\mathcal{O}\!\big(L\,\tfrac{\widetilde{S}}{1-\rho}\big)$.
\end{enumerate}
\end{theorem}

\begin{proof}
[Proof of Theorem~{\upshape\ref{app-thm: emergent potential equals non-predictability}}]
Identical to Theorem~\ref{app-thm: neuronal stability measures reasoning capacity}, since only the disturbance term changes while the contraction modulus $\rho$ is preserved. 
Items (ii)–(iii) are immediate by monotonicity in $\widetilde{S}$.
\end{proof}

\subsection{Supporting Propositions}



\begin{proposition}[Generalized intra-group communication)]
\label{app-prop: C_ij via SVD}
Let $\mathbf{W}\in\mathbb{R}^{m\times n}$ be any linear map. Consider intra-group communication governed by
\[
y_i \;=\; \sum_{j=1}^n \mu(\mathbf{a}_i,\mathbf{b}_j)\,x_j
\;=\; \sum_{j=1}^n \big\langle \phi_L(\mathbf{a}_i),\,\phi_R(\mathbf{b}_j)\big\rangle\, x_j,
\]
where $\{\mathbf{a}_i\}_{i=1}^m$ and $\{\mathbf{b}_j\}_{j=1}^n$ are output- and input-side neuronal states, respectively, and
\[
\phi_L(\mathbf{q}) := \sigma\!\big(\mathbf{q}\,\mathbf{G}_{\mathrm{left}}\big)\in\mathbb{R}^r,\qquad
\phi_R(\mathbf{p}) := \sigma\!\big(\mathbf{p}\,\mathbf{G}_{\mathrm{right}}\big)\in\mathbb{R}^r,
\]
with $\sigma$ a component-wise activation (identity map or a non-polynomial nonlinearity such as $\tanh$). 
Then:
\begin{enumerate}
\item[\emph{(i)}] \textbf{Exact realization when $\sigma=\mathrm{id}$.} There exist $r\le\mathrm{rank}(\mathbf{W})$, matrices $\mathbf{A}:=[\mathbf{a}_1^\top;\ldots;\mathbf{a}_m^\top]\in\mathbb{R}^{m\times r}$ and $\mathbf{B}:=[\mathbf{b}_1^\top;\ldots;\mathbf{b}_n^\top]\in\mathbb{R}^{n\times r}$, and $\mathbf{G}_{\mathrm{left}},\mathbf{G}_{\mathrm{right}}\in\mathbb{R}^{r\times r}$ such that
\[
\mathbf{W}_{ij} \;=\; \big\langle \mathbf{a}_i\mathbf{G}_{\mathrm{left}},\, \mathbf{b}_j\mathbf{G}_{\mathrm{right}}\big\rangle
\quad\text{for all }i,j.
\]
Consequently, $y=\mathbf{W}x$ is realized exactly by intra-group communication with $\mu(\mathbf{q},\mathbf{p})=\big\langle \mathbf{q}\mathbf{G}_{\mathrm{left}},\,\mathbf{p}\mathbf{G}_{\mathrm{right}}\big\rangle$.
\item[\emph{(ii)}] \textbf{Nonlinear feature maps preserve expressivity.} If $\sigma$ is any non-polynomial activation (e.g., $\tanh$), then for every $\varepsilon>0$ there exist a width $r$ and parameters $(\mathbf{A},\mathbf{B},\mathbf{G}_{\mathrm{left}},\mathbf{G}_{\mathrm{right}})$ such that
\[
\max_{1\le i\le m,\;1\le j\le n}
\Big|\,\mathbf{W}_{ij} - \big\langle \phi_L(\mathbf{a}_i),\,\phi_R(\mathbf{b}_j)\big\rangle\,\Big| \;<\; \varepsilon.
\]
Thus the nonlinear $\mu$ can approximate any linear block $\mathbf{W}$ arbitrarily well on the finite neuronal state sets used by the group.
\end{enumerate}
\end{proposition}

\begin{proof}[Proof of Proposition~{\upshape\ref{app-prop: C_ij via SVD}}]
We conduct the proof in the case of $\sigma=\mathrm{id}$ and $\sigma$ is any non-polynomial activation, respectively.

\emph{(i)} With $\sigma=\mathrm{id}$, $\mu(\mathbf{q},\mathbf{p})=\mathbf{q}\,\mathbf{G}_{\mathrm{left}}\mathbf{G}_{\mathrm{right}}^{\top}\mathbf{p}^{\top}$ is a general bilinear pairing. Take an SVD $\mathbf{W}=\mathbf{U}\boldsymbol{\Sigma}\mathbf{V}^\top$, define $r=\mathrm{rank}(\mathbf{W})$, set $\mathbf{A}:=\mathbf{U}\boldsymbol{\Sigma}^{1/2}$ and $\mathbf{B}:=\mathbf{V}\boldsymbol{\Sigma}^{1/2}$, and choose $\mathbf{G}_{\mathrm{left}}=\mathbf{I}$ and $\mathbf{G}_{\mathrm{right}}=\mathbf{I}$. Then
\[
\big\langle \mathbf{a}_i\mathbf{G}_{\mathrm{left}},\, \mathbf{b}_j\mathbf{G}_{\mathrm{right}}\big\rangle
= \big\langle \mathbf{A}[i],\,\mathbf{B}[j]\big\rangle
= (\mathbf{U}\boldsymbol{\Sigma}\mathbf{V}^\top)_{ij}
= \mathbf{W}_{ij},
\]
so $y=\mathbf{W}x$ is realized exactly.

\emph{(ii)} Fix any factorization $\mathbf{W}_{ij}=u_i^\top v_j$ with $\{u_i\},\{v_j\}\subset\mathbb{R}^r$ (e.g., from the SVD above with $u_i=\mathbf{A}[i]$, $v_j=\mathbf{B}[j]$). By the universal approximation property of single-hidden-layer networks with non-polynomial activations, for any finite sets $\{\mathbf{a}_i\}$ and $\{\mathbf{b}_j\}$ and any $\varepsilon>0$ we can choose width $r$ and $\mathbf{G}_{\mathrm{left}},\mathbf{G}_{\mathrm{right}}$ so that
$\phi_L(\mathbf{a}_i)\approx u_i$ and $\phi_R(\mathbf{b}_j)\approx v_j$ uniformly within error $<\delta$, which makes
$\langle \phi_L(\mathbf{a}_i),\phi_R(\mathbf{b}_j)\rangle$ approximate $u_i^\top v_j=\mathbf{W}_{ij}$ within $<\varepsilon$ for all $(i,j)$ by continuity of the inner product.
\end{proof}

\begin{proposition}[Continuous-to-discrete reduction with data-driven projections and an ISS-type bound]\label{app-prop: simplified-dynamical-sys}
Under mild excitation of data (diverse activations), the continuous dynamics (Eq.~\ref{eq: raw-dynamical-sys}) admits a least-squares (data-driven) discretization of the form (Eq.~\ref{eq: simplified-dynamical-sys}) with projections $\mathcal{T}$ and $\mathcal{H}$ obtained by minimizing the residuals in Eq.~\ref{eq: define-act-residual}.
If $\max_{x,\dagger}\{\lVert\mathcal{T}_{x,\dagger}^{(com)}\rVert,\lVert\mathcal{H}_{x,\dagger}^{(com)}\rVert\}<1$, the discrete system enjoys the ISS bound in Theorem~\ref{app-thm: intra/inter-group coms are sufficient}.

\end{proposition}

\begin{proof}[Proof of Proposition~{\upshape\ref{app-prop: simplified-dynamical-sys}}]
Regressing $\mathcal{A}_{t+1}^{(\text{root})}$ and $\mathcal{A}_{t}^{(\text{com})}$ onto $\mathcal{A}_{t+1}^{(\text{com})}$ (and analogously for $\mathcal{Q}$) yields Eq.~\ref{eq: simplified-dynamical-sys} with residuals defined in Eq.~\ref{eq: define-act-residual}.
Contractivity of $\mathcal{T}^{(com)}$ and $\mathcal{H}^{(com)}$ then implies ISS by standard linear time-varying stability arguments.

\end{proof}

\begin{proposition}[lossless state sharing with nonlinear $\mu$]
\label{app-prop: overlapping neurons}
Let two groups $G_i$ and $G_j$ act on the \emph{same} input-side neuronal states $\{\mathbf{b}_\ell\}$ (i.e., input-side \emph{state sharing}). Suppose both groups use the same right-feature map
\[
\phi_R(\mathbf{p})=\sigma\!\big(\mathbf{p}\,\mathbf{G}_{\mathrm{right}}\big)\in\mathbb{R}^r,
\]
but (potentially) different left-feature maps
\[
\phi_L^{(i)}(\mathbf{q})=\sigma\!\big(\mathbf{q}\,\mathbf{G}_{\mathrm{left}}^{(i)}\big),\qquad
\phi_L^{(j)}(\mathbf{q})=\sigma\!\big(\mathbf{q}\,\mathbf{G}_{\mathrm{left}}^{(j)}\big).
\]
Let $\{\mathbf{a}^{(i)}_r\}$ and $\{\mathbf{a}^{(j)}_r\}$ be the output-side rows for $G_i$ and $G_j$ respectively.

\noindent
\textbf{Claim.} Sharing the input-side states is \emph{lossless} (i.e., the two groups compute identical input–output maps for all inputs) if and only if there exists a feature-space linear isometry $\mathbf{U}\in\mathbb{R}^{r\times r}$ (orthogonal/unitary) such that, on the relevant state support,
\[
\phi_L^{(j)}\!\big(\mathbf{a}^{(j)}_r\big) \;=\; \phi_L^{(i)}\!\big(\mathbf{a}^{(i)}_r\big)\,\mathbf{U}
\quad\text{for every output row }r.
\]
Equivalently, the two families of left features differ only by a basis change that preserves inner products. In that case,
\[
\big\langle \phi_L^{(j)}\!\big(\mathbf{a}^{(j)}_r\big),\, \phi_R(\mathbf{b}_\ell)\big\rangle
 \;=\; \big\langle \phi_L^{(i)}\!\big(\mathbf{a}^{(i)}_r\big),\, \phi_R(\mathbf{b}_\ell)\big\rangle
\quad\text{for all }r,\ell.
\]
Conversely, if functional equality holds on a set for which $\mathrm{span}\{\phi_R(\mathbf{b}_\ell)\}=\mathbb{R}^r$, then the two left-feature families must be related by such an isometry on that span.
\end{proposition}

\begin{proof}[Proof of Proposition~{\upshape\ref{app-prop: overlapping neurons}}]
We conduct the proof by sufficiency and necessity, respectively.

\emph{(Sufficiency).} If $\phi_L^{(j)}(\cdot)=\phi_L^{(i)}(\cdot)\,\mathbf{U}$ with $\mathbf{U}$ an isometry, then for all $\mathbf{u},\mathbf{v}\in\mathbb{R}^r$,
$\langle \mathbf{u}\mathbf{U},\,\mathbf{v}\rangle=\langle \mathbf{u},\,\mathbf{U}^\top\mathbf{v}\rangle=\langle \mathbf{u},\,\mathbf{v}\rangle$,
so the inner products with the \emph{shared} right features $\phi_R$ coincide, yielding identical outputs.

\emph{(Necessity).} Suppose
$\langle \phi_L^{(j)}(\mathbf{a}),\, \phi_R(\mathbf{b})\rangle
=\langle \phi_L^{(i)}(\mathbf{a}'),\, \phi_R(\mathbf{b})\rangle$
for all shared inputs $\mathbf{b}$ and all output rows. Restricting to the linear span $\mathcal{S}:=\mathrm{span}\{\phi_R(\mathbf{b}_\ell)\}$, both left-feature families induce the same linear functionals on $\mathcal{S}$. 
If $\mathcal{S}=\mathbb{R}^r$, the Riesz representation in a finite-dimensional inner-product space implies the two representing vectors must be related by an inner-product-preserving linear map $\mathbf{U}$, {\it i.e.}, an isometry.
\end{proof}

\begin{proposition}[Symmetry and reparameterization under nonlinear $\mu$]
\label{app-prop: intra-group metric is riemannian}
Let $\mu(\mathbf{q},\mathbf{p})=\big\langle \phi_L(\mathbf{q}),\,\phi_R(\mathbf{p})\big\rangle$ with
\[
\phi_L(\mathbf{q})=\sigma\!\big(\mathbf{q}\,\mathbf{G}_{\mathrm{left}}\big),\qquad
\phi_R(\mathbf{p})=\sigma\!\big(\mathbf{p}\,\mathbf{G}_{\mathrm{right}}\big)\in\mathbb{R}^r.
\]
Then, we have:
\begin{enumerate}
\item[\emph{(i)}] \textbf{No-loss reduction to a Euclidean dot product.}
There exist invertible linear maps $\mathbf{C}_L,\mathbf{C}_R\in\mathbb{R}^{r\times r}$ such that, after the feature re-parameterizations
\[
\widetilde{\phi}_L:=\phi_L\,\mathbf{C}_L,\qquad
\widetilde{\phi}_R:=\phi_R\,\mathbf{C}_R,
\]
we have
\[
\mu(\mathbf{q},\mathbf{p})
=\big\langle \widetilde{\phi}_L(\mathbf{q}),\,\widetilde{\phi}_R(\mathbf{p})\big\rangle,
\quad\text{with}\quad
\mathbf{C}_L^\top \mathbf{C}_R=\boldsymbol{\Sigma}\ \text{diagonal},\ \boldsymbol{\Sigma}\succeq \mathbf{0}.
\]
For any bilinear coupling $\langle \mathbf{u}, \mathbf{M}\mathbf{v}\rangle$ between the two feature spaces, take an SVD $\mathbf{M}=\mathbf{U}\boldsymbol{\Sigma}\mathbf{V}^\top$ and absorb $\mathbf{U},\mathbf{V}$ into $\mathbf{C}_L,\mathbf{C}_R$ to obtain an ordinary dot product in the reparameterized features.
\item[\emph{(ii)}] \textbf{Symmetric Riemannian metric.}
If we tie $\phi_L=\phi_R=\phi$ and endow the feature space with an SPD metric $\langle \mathbf{u},\mathbf{v}\rangle_{\mathbf{G}} := \mathbf{u}^\top \mathbf{G}\,\mathbf{v}$ (with $\mathbf{G}=\mathbf{G}^\top\succ \mathbf{0}$), then this defines an inner product (hence a Riemannian metric) on the feature manifold that is invariant to orthonormal basis changes $\mathbf{u}\mapsto \mathbf{U}\mathbf{u}$ ($\mathbf{U}$ orthogonal/unitary). 
Conversely, since $\mathbf{G}=\mathbf{R}^\top \mathbf{R}$ (e.g., Cholesky/polar), the change of variables $\mathbf{u}\mapsto \mathbf{R}\mathbf{u}$ reduces $\langle \cdot,\cdot\rangle_{\mathbf{G}}$ to the standard Euclidean dot product. Therefore, one may assume $\mathbf{G}=\mathbf{I}$ \emph{without loss of representational power}.
\end{enumerate}
\end{proposition}

\begin{proof}
[Proof of Proposition~{\upshape\ref{app-prop: intra-group metric is riemannian}}]
We conduct the proof in the following two cases.

\emph{(i)} Any bilinear form between finite-dimensional feature spaces can be written as $\langle \mathbf{u}, \mathbf{M}\mathbf{v}\rangle$. 
Taking the SVD $\mathbf{M}=\mathbf{U}\boldsymbol{\Sigma}\mathbf{V}^\top$ with $\boldsymbol{\Sigma}\succeq \mathbf{0}$ yields
\[
\langle \mathbf{u}, \mathbf{M}\mathbf{v}\rangle
= \big\langle \mathbf{U}^\top\mathbf{u},\, \boldsymbol{\Sigma}\,\mathbf{V}^\top\mathbf{v}\big\rangle
= \big\langle \boldsymbol{\Sigma}^{1/2}\mathbf{U}^\top\mathbf{u},\, \boldsymbol{\Sigma}^{1/2}\mathbf{V}^\top\mathbf{v}\big\rangle,
\]
which is a Euclidean dot product after setting $\mathbf{C}_L:=\mathbf{U}\boldsymbol{\Sigma}^{1/2}$ and $\mathbf{C}_R:=\mathbf{V}\boldsymbol{\Sigma}^{1/2}$.

\emph{(ii)} If $\mathbf{G}\succ\mathbf{0}$, then $\langle \cdot,\cdot\rangle_{\mathbf{G}}$ is an inner product and defines a Riemannian metric. Orthonormal basis changes preserve inner products, giving invariance. Conversely, with $\mathbf{G}=\mathbf{R}^\top\mathbf{R}$ (Cholesky), we have $\mathbf{u}^\top\mathbf{G}\mathbf{v}=\langle \mathbf{R}\mathbf{u},\,\mathbf{R}\mathbf{v}\rangle$, so working in the coordinates $\mathbf{R}\mathbf{u}$ reduces to the Euclidean case, implying the ``without loss of generality'' claim for $\mathbf{G}=\mathbf{I}$.
\end{proof}

\end{appendices}

\end{document}